\documentclass[final,onefignum,onetabnum]{siamonline220329}

\usepackage{braket,amsfonts}

\usepackage{array}
\usepackage[caption=false]{subfig}

\usepackage{pgfplots}

\newsiamremark{remark}{Remark}
\newsiamremark{hypothesis}{Hypothesis}
\crefname{hypothesis}{Hypothesis}{Hypotheses}
\newsiamthm{claim}{Claim}

\usepackage{enumitem}
\setlist[enumerate]{leftmargin=.5in}
\setlist[itemize]{leftmargin=.5in}

\usepackage{algorithmic}

\usepackage{graphicx,epstopdf}
\usepackage{amsfonts}
\ifpdf
  \DeclareGraphicsExtensions{.eps,.pdf,.png,.jpg}
\else
  \DeclareGraphicsExtensions{.eps}
\fi

\Crefname{ALC@unique}{Line}{Lines}

\usepackage{amsopn}

\usepackage{xspace}
\usepackage{bold-extra}
\usepackage[most]{tcolorbox}

\colorlet{texcscolor}{blue!50!black}
\colorlet{texemcolor}{red!70!black}
\colorlet{texpreamble}{red!70!black}
\colorlet{codebackground}{black!25!white!25}

\lstdefinestyle{siamlatex}{%
  style=tcblatex,
  texcsstyle=*\color{texcscolor},
  texcsstyle=[2]\color{texemcolor},
  keywordstyle=[2]\color{texemcolor},
  moretexcs={cref,Cref,maketitle,mathcal,text,headers,email,url},
}

\tcbset{%
  colframe=black!75!white!75,
  coltitle=white,
  colback=codebackground, 
  colbacklower=white, 
  fonttitle=\bfseries,
  arc=0pt,outer arc=0pt,
  top=1pt,bottom=1pt,left=1mm,right=1mm,middle=1mm,boxsep=1mm,
  leftrule=0.3mm,rightrule=0.3mm,toprule=0.3mm,bottomrule=0.3mm,
  listing options={style=siamlatex}
}

\newtcblisting[use counter=example]{example}[2][]{%
  title={Example~\thetcbcounter: #2},#1}

\newtcbinputlisting[use counter=example]{\examplefile}[3][]{%
  title={Example~\thetcbcounter: #2},listing file={#3},#1}

\DeclareTotalTCBox{\code}{ v O{} }
{ 
  fontupper=\ttfamily\color{black},
  nobeforeafter,
  tcbox raise base,
  colback=codebackground,colframe=white,
  top=0pt,bottom=0pt,left=0mm,right=0mm,
  leftrule=0pt,rightrule=0pt,toprule=0mm,bottomrule=0mm,
  boxsep=0.5mm,
  #2}{#1}

\patchcmd\newpage{\vfil}{}{}{}
\usepackage{amsmath,amsfonts,amssymb}
\usepackage{bm}
\usepackage{mathtools}
\usepackage{ntheorem}
\usepackage{accents}
\usepackage[numbers,sort]{natbib}
\usepackage{bibentry}
\usepackage{adjustbox}
\usepackage{booktabs}
\usepackage{multirow}
\usepackage{ulem}
\newtheorem{assumption}[theorem]{Assumption}

\newcommand{\e}{{\mathrm e}}
\renewcommand{\i}{{\mathrm i}}
\newcommand{\dx}{{\,\mathrm dx}}
\newcommand{\dbmx}{{\,\mathrm d\bm x}}

\renewcommand\intercal{\mathsf{\scriptscriptstyle T}}

\usepackage[mathlines]{lineno}

\begin{tcbverbatimwrite}{tmp_\jobname_header.tex}
\title{Preconditioned Additive Gaussian Processes with Fourier Acceleration\thanks{Submitted to the editors DATE.
\funding{Research of Y. Xi is supported by NSF DMS-2338904.}}}

\author{Theresa Wagner\thanks{Department of Mathematics, Chemnitz University of Technology, Germany (\email{theresa.wagner@math.tu-chemnitz.de},\email{franziska.nestler@math.tu-chemnitz.de},\email{martin.stoll@math.tu-chemnitz.de})}
\and Tianshi Xu\thanks{Department of Mathematics, Emory University, Atlanta, GA (\email{tianshi.xu@emory.edu}, \email{yuanzhe.xi@emory.edu}).}
\and Franziska Nestler\footnotemark[2] \and Yuanzhe Xi\footnotemark[3] \and Martin Stoll\footnotemark[2]}

\headers{Preconditioned Additive Gaussian Processes with Fourier Acceleration}{Theresa Wagner, Tianshi Xu, Franziska Nestler, Yuanzhe Xi and Martin Stoll}
\end{tcbverbatimwrite}
\input{tmp_\jobname_header.tex}

\ifpdf
\hypersetup{ pdftitle={Preconditioned Additive Gaussian Processes with Fourier Acceleration} }
\fi

\begin{document}
\maketitle

\begin{tcbverbatimwrite}{tmp_\jobname_abstract.tex}
\begin{abstract}
Gaussian processes (GPs) are crucial in machine learning for quantifying uncertainty in predictions. However, their associated covariance matrices, defined by kernel functions, are typically dense and large-scale, posing significant computational challenges. This paper introduces a matrix-free method that utilizes the Non-equispaced Fast Fourier Transform (NFFT) to achieve nearly linear complexity in the multiplication of kernel matrices and their derivatives with vectors for a predetermined accuracy level. To address high-dimensional problems, we propose an additive kernel approach. Each sub-kernel in this approach captures lower-order feature interactions, allowing for the efficient application of the NFFT method and potentially increasing accuracy across various real-world datasets. Additionally, we implement a preconditioning strategy that accelerates hyperparameter tuning, further improving the efficiency and effectiveness of GPs.
\end{abstract}

\begin{keywords}
Gaussian process, additive kernel, NFFT, preconditioning, error analysis
\end{keywords}

\begin{MSCcodes}
65C60, 65D15, 65F10, 65T50  
\end{MSCcodes}
\end{tcbverbatimwrite}
\input{tmp_\jobname_abstract.tex}
\section{Introduction}\label{sec:intro}
\textit{Gaussian processes} (GPs) model distributions over function evaluations \(\phi(\bm{x})\), characterized by a mean function \(m(\bm{x})\) and a covariance function \(\kappa: \mathbb R^p\times\mathbb R^p\to\mathbb R\)~\cite{williams2006gaussian}. More specifically, GPs assume that the function values \(\phi(\bm{x}_j)\) at any finite collection of points \(\mathcal{X} = \{\bm{x}_1, \ldots, \bm{x}_n\}\subset\mathbb R^p\) follow a Gaussian distribution, represented as \(\rho(\bm{\phi} | \mathcal{X}) = \mathcal{N}(\bm{\phi} | \bm{\mu}, K)\).
Here, \(\bm{\phi} = [\phi(\bm{x}_1), \ldots, \phi(\bm{x}_n)]^\intercal\), \(\bm{\mu} = [m(\bm{x}_1), \ldots, m(\bm{x}_n)]^\intercal\), and \(K\) is the covariance matrix with elements \(K_{ij} = \kappa(\bm{x}_i, \bm{x}_j)\) for \(i, j = 1, \ldots, n\). Assuming the observed outcomes \(Y\) are expressed as \(Y = \bm{\phi} + \bm{\varepsilon}\), with \(\bm{\varepsilon}\) being Gaussian noise characterized by covariance \(\sigma_\varepsilon^2I\), then the conditional distribution of \(Y\) given \(\bm{\phi}\) is modeled as \(Y | \bm{\phi} \sim \mathcal{N}(\bm{\phi}, \sigma_\varepsilon^2I)\).

Two widely used covariance functions include the Gaussian or RBF kernel $\kappa^{\text{g}}$ and the Mat\'ern\texorpdfstring{\((\frac{1}{2})\)}{} kernel $\kappa^{\text{m}} $, defined as follows:
\begin{align} \label{eq:kernel_definition}
    \kappa^{\text{g}} (\bm{x}_i,\bm{x}_j) =  \sigma_f^2 \exp \left( -\tfrac{\|\bm{x}_i-\bm{x}_j\|_2^2}{2\ell^2} \right), \quad \kappa^{\text{m}} (\bm{x}_i,\bm{x}_j) =  \sigma_f^2 \exp \left( -\tfrac{\|\bm{x}_i-\bm{x}_j\|_2}{\ell} \right),
\end{align}
where \(\ell > 0\) represents the length scale, and \(\sigma_f > 0\) indicates the prior variance. Consequently, the GP model hyperparameters, denoted by \(\bm{\theta} = (\sigma_f, \ell, \sigma_\varepsilon)\), encompass the variance scales and the noise level, central to defining the behavior of these kernels. Given that both kernels are shift-invariant, we simplify the notation by representing them with a single input in the analysis:
\begin{equation*}
  \kappa(\bm x-\bm y)  \coloneqq  \kappa(\bm x,\bm y).
\end{equation*}
The GP objective function, aimed at determining the optimal hyperparameters \(\bm{\theta}\), is the negative log marginal likelihood:
\begin{align} \label{eq:GP_objective}
    Z(\bm{\theta}) = -\log \rho(Y | \mathcal{X}) = \tfrac{1}{2} \left(Y^\intercal \hat{K}^{-1} Y + \log(\det(\hat{K})) + n \log(2 \pi) \right),
\end{align}
where \(\hat{K} = K + \sigma_\varepsilon^2I\) is referred to as the regularized kernel matrix. Given the large size of the kernel matrix \(K\), direct computations of \(Y^\intercal \hat{K}^{-1} Y\) and \(\log(\det(\hat{K}))\) are considered infeasible for evaluating \(Z(\bm{\theta})\). Thus, iterative methods are invoked, where the preconditioned Conjugate Gradient (CG) algorithm is used to approximate \(\hat{K}^{-1} Y\)~\cite{williams2006gaussian}, and the Hutchinson trace estimator is employed to approximate \(\log(\det(\hat{K}))\) via
\begin{align*}
\log(\det(\hat{K})) = \operatorname{tr}(\text{logm}(\hat{K})) \approx \frac{1}{n_z} \sum_{i=1}^{n_z} \bm{z}_i^\intercal \text{logm}(\hat{K}) \bm{z}_i,
\end{align*}
where \(\text{logm}\) denotes the matrix logarithm and \(\bm{z}_i \in \mathbb{R}^n\), \(i=1, \ldots, n_z\), are random Rademacher vectors~\cite{hutchinson1989stochastic}. The quadratic terms $\{\bm{z}_i^\intercal \text{logm}(\hat{K}) \bm{z}_i\}$ can be further approximated through the Lanczos algorithm, which is known as the \textit{stochastic Lanczos quadrature} (SLQ) method~\cite{ubaru2017fast}. 

If \(M\), an approximation of \(\hat{K}\), is available, then \(\log(\det(\hat{K}))\) can be decomposed as:
\begin{align} \label{eq:log_det_decomp}
    \log(\det(\hat{K})) = \log(\det(M)) +\operatorname{tr}{\underbrace{(\text{logm} (\hat{K})-\text{logm} (M))}_{=\Delta \log}}.
\end{align}
When \(\log(\det(M))\) is explicitly computable, the SLQ is only applied to estimate \(\operatorname{tr}(\Delta \log)\) in this case, which has been demonstrated to converge more rapidly when \(M\) is a good preconditioner for \(\hat{K}\)~\cite{wenger2022preconditioning}. Therefore, in this paper, we will optimize the GP hyperparameters \(\bm{\theta}\) by maximizing the following preconditioned approximate objective function \(\tilde{Z}(\bm{\theta})\):
\begin{align} \label{eq:GP_approx_objective}
    \tilde{Z}(\bm{\theta}) = \tfrac{1}{2} \left( Y^\intercal \hat{K}^{-1} Y + \log(\det(M)) + \tfrac{1}{n_z} \sum_{i=1}^{n_z} \bm{z}_i^\intercal \text{logm}(M^{-1}\hat{K}) \bm{z}_i + n \log (2 \pi) \right) \approx Z(\bm{\theta}),
\end{align}
and the choice of the preconditioner $M$ will be discussed in Section~\ref{subsec:Preconditioning}.

Optimizing \(\bm{\theta}\) using first-order optimization methods also requires computing the derivatives of \(\tilde{Z}(\bm{\theta})\). These derivatives can be further approximated as follows
\begin{equation}
\frac{\partial \tilde{Z}(\bm{\theta})}{\partial \theta_j} \approx \tfrac{1}{2} \left(-\bm{\alpha}^\intercal \frac{\partial \hat{K}}{\partial \theta_j} \bm{\alpha} + \operatorname{tr}(M^{-1}\frac{\partial M}{\partial \theta_j}) + \tfrac{1}{n_z} \sum_{i=1}^{n_z} \bm{z}_i^\intercal (M^{-1}\hat{K})^{-1} \frac{\partial (M^{-1}\hat{K})}{\partial \theta_j} \bm{z}_i\right),
\label{eq:div}
\end{equation}
where \(\bm{\alpha}\) is the solution to the linear system \(\hat{K} \bm{\alpha} = Y\) \cite{wenger2022preconditioning}. When employing the preconditioned CG and SLQ to approximate the objective function and its derivatives, the primary computational costs stem from matrix-vector multiplications with the kernel matrix and its derivatives. These operations form the main computational bottlenecks in the GP optimization process.

In this paper, we introduce a novel preconditioned additive GP model that leverages the Non-equispaced Fast Fourier Transform (NFFT) to accelerate the matrix-vector multiplication operations and utilizes preconditioning to improve the convergence of iterative methods for speeding up the GP hyperparameter optimization. The remaining sections are organized as follows: Section~\ref{sec:Additive_GPs} details the new additive GP model structure. Section~\ref{sec:Fourier_Accelerated} introduces the NFFT for accelerating kernel matrix-vector multiplications and its rigorous approximation error analysis is provided in Section~\ref{sec:Theoretical_Analysis}. Numerical examples illustrating the effectiveness of the proposed approach are presented in Section~\ref{sec:Experiments}, followed by concluding remarks in Section~\ref{sec:Conclusion}.

\section{Additive Gaussian Processes} \label{sec:Additive_GPs}

A core principle of GPs is that data points proximate in the input space typically yield similar outputs. By partitioning the high-dimensional feature space into smaller subspaces and applying an additive structure to these segments, one can enhance the relevance of distance and neighborhood relations \cite{durrande2011additive, durrande2012additive}. In recent years, there has been a growing interest in the additive kernel and multiple kernel learning for GPs \cite{durrande2011additive, durrande2012additive, duvenaud2011additive}. For instance, GP models have been effectively combined with Generalized Additive Models (GAMs) \cite{hastie1986generalized} as demonstrated in \cite{durrande2012additive}. This integration shows that additive models can accurately represent the additive characteristics of functions, even in cases where the underlying regression function is not inherently additive.
Similarly, the covariance kernel of GPs has been augmented with additivity as described in \cite{durrande2011additive}. In this enhancement, the response of the GAM simulator is approximated through a sum of univariate functions, which improves both the interpretability and predictive accuracy of the model.
Moreover, additive GP regression has been demonstrated to achieve \textit{near minimax-optimal error rates} for the additive function class \cite{jiang2021variable,yang2015minimax}. Finally, from a computational point of view, this approach facilitates the use of fast matrix-vector multiplication algorithms such as NFFT or hierarchical matrix methods \cite{ddh2,huang2020h2pack,huang2025higp}, taking advantage of the reduced dimensionality to optimize computational efficiency.

\subsection{Additive Kernels} \label{subsec:Additive_kernels}
In this paper, we consider an additive kernel structure given by \(K = \sigma_f^2(K_1 + \dots + K_P)\), where each \(K_s \in \mathbb{R}^{n \times n}\) for \(s = 1, \dots, P\) represents a distinct \textit{sub-kernel} without prior variance term and \(\sigma_f\) is a scaling factor applied uniformly across \(P\) sub-kernels. 
The kernel function characterizing the covariance matrix is defined additively as
\begin{align} \label{eq:additive_kernel}
    \kappa (\bm{x}_i, \bm{x}_j) = \sigma_f^2 \sum_{s=1}^P \kappa_s \left(\bm{x}_i^{\mathcal{W}_s}, \bm{x}_j^{\mathcal{W}_s} \right),
\end{align}
where \(\kappa_s\) defines the sub-kernels based on subsets of features determined by the index set  \(\mathcal{W}_s \subset \{1,2,\ldots, p\}\) when the data points \(\bm{x}_i\) belong to \(\mathbb{R}^{p}\). Here, each \(\mathcal{W}_s\) contains \(d_s\) indices, ensuring that \(\mathcal{W}_{s_1} \cap \mathcal{W}_{s_2} = \emptyset\) for \(s_1 \neq s_2\) and \(\bm{x}_i^{\mathcal{W}_s}\) only contains the features corresponding to the indices in $\mathcal{W}_s$. The total number of features used in the additive model satisfies \(\sum_{s=1}^P d_s \leq p\), promoting dimensionality reduction within the model. This reduction can be implemented by either selecting features based on a threshold \(\texttt{thres} > 0\) or a feature importance ratio \(d_{\text{ratio}} \in (0,1)\), which drops features scoring below \(\texttt{thres}\) or outside the top \(d_{\text{ratio}}\) proportion in a feature importance ranking algorithm. Further details and empirical studies on the impact of \(\texttt{thres} \) and \(d_{\text{ratio}}\) on model performance can be found in Section~\ref{subsec:GP_framework}. In the following, we refer to the set of feature indices $\mathcal{W}\coloneqq\left[\mathcal{W}_1, \dots,\mathcal{W}_P\right]$ as \textit{feature windows}. 

Since the sum of any two positive-definite kernels is also positive-definite \cite{murphy2023probabilistic}, the additive kernel in \eqref{eq:additive_kernel}, constructed using Gaussian or Mat\'ern\texorpdfstring{\((\frac{1}{2})\)}{} sub-kernels, is itself a valid positive-definite Mercer kernel. Moreover, as the definitions of the GP posterior mean and variance hold for any symmetric positive-definite kernel, they naturally extend to additive kernels. Consequently, we define the \textit{additive Gaussian kernel} and the \textit{additive Mat\' ern\texorpdfstring{$(\frac{1}{2})$}{} kernel} functions as follows
\begin{align} \label{eq:def_additive_kernel}
    \kappa^{\text{G}} ( \bm{x}_i , \bm{x}_j ) = \sigma_{f}^2 \sum_{s=1}^P \underbrace{\exp \left( - \tfrac{\| \bm{x}_i^{\mathcal{W}_s} - \bm{x}_j^{\mathcal{W}_s} \|_2^2}{2 \ell^2} \right)}_{\eqqcolon\kappa_s^{\text{G}}}, \quad \kappa^{\text{M}} ( \bm{x}_i , \bm{x}_j ) = \sigma_{f}^2 \sum_{s=1}^P \underbrace{\exp \left( - \tfrac{\| \bm{x}_i^{\mathcal{W}_s} - \bm{x}_j^{\mathcal{W}_s} \|_2}{ \ell} \right)}_{\eqqcolon\kappa_s^{\text{M}}}
\end{align}
where \(\kappa_s^{\text{G}}\) and \(\kappa_s^{\text{M}}\) are the windowed Gaussian and Mat\' ern\texorpdfstring{$(\frac{1}{2})$}{} kernels. 
The derivatives of the resulting regularized additive kernel $\hat{K} = \sigma_{f}^2(K_1+\dots+K_P)+\sigma_{\varepsilon}^2I$ with respect to $\sigma_f$ and $\sigma_{\varepsilon}$ are straightforward. The derivatives with respect to the length-scale parameter $\ell$ are defined as $K^{\text{derG}} = \sigma_f^2 \sum_{s=1}^P K_s^{\text{derG}}$ and $K^{\text{derM}} = \sigma_f^2 \sum_{s=1}^P K_s^{\text{derM}}$, with
\begin{align} \label{eq:def_der_additive_kernels}
    \kappa_s^{\text{derG}} (\bm{x}_i, \bm{x}_j) = \tfrac{ \| \bm{x}_i^{\mathcal{W}_s} - \bm{x}_j^{\mathcal{W}_s} \|_2^2}{\ell^3} \kappa_s^{\text{G}} (\bm{x}_i, \bm{x}_j), \quad \kappa_s^{\text{derM}} (\bm{x}_i, \bm{x}_j) = \tfrac{ \| \bm{x}_i^{\mathcal{W}_s} - \bm{x}_j^{\mathcal{W}_s} \|_2}{\ell^2} \kappa_s^{\text{M}} (\bm{x}_i, \bm{x}_j).
\end{align}

\subsection{Feature Grouping Techniques} \label{subsec:Feature_Arrangement}
In the context of additive kernels, feature grouping involves dividing the entire set of features into smaller subsets called feature groups, denoted as $\mathcal{W}_s$. Each group $\mathcal{W}_s$ consists of $d_s$ features where $d_s$ is bounded by the maximum group size, $d_{\text{max}}$, defined as $d_{\text{max}} = 3$ in this paper. For example, consider $\mathcal{W}_s = \{a, b, c\}$ within the range $\{1, \dots, p\}$, then for $d_s = 3$, the data points restricted to these feature indices are represented as $\bm{x}_i^{\mathcal{W}_s} = \left[x_{i_a}, x_{i_b}, x_{i_c} \right]^\intercal$. This grouping can be based on feature similarity, although similar features may also be distributed across different groups to explore diverse interactions. A comprehensive analysis and comparison of several existing feature grouping techniques, specifically for the additive kernel setting, can be found in \cite{wagner2024fast}.

In this paper, we apply two techniques for determining the feature windows. The first technique ranks the feature's importance based on each feature's mutual information score (MIS)~\cite{battiti1994using}. MIS is a univariate measure that quantifies how much information about the label can be obtained by knowing the feature value. Based on the obtained feature importance scores, the features are ranked in descending order and grouped consecutively into groups of the desired size. Alternatively, elastic-net (EN)~\cite{zou2005regularization} regression can be employed for determining the feature windows. EN is a regression model that is based on a least-squares and two regularization terms (\(L_1\)- and \(L_2\)-norm of the coefficient vector \(\bm{w}\in\mathbb{R}^p\)) for enforcing sparsity into the model so that most coefficients will be zero and the corresponding features are dropped. The corresponding objective \(Z_{\text{EN}} = \tfrac{1}{2n} \| X\bm{w} - Y \|_2^2 + \lambda_{\text{EN}} \rho \|\bm{w} \|_1 + \tfrac{\lambda_{\text{EN}} (1-\rho)}{2} \|\bm{w} \|_2\) is minimized with respect to the coefficients \(\bm{w}\), and the ratio between the penalty terms is balanced with the ratio $\rho$. For \(\rho=1\), EN equals the well-known Lasso regularization~\cite{tibshirani1996regression}.
Once the sparse coefficient vector is obtained, the entries of $\bm{w}$ can be used as feature importance scores based on which the features can be assigned to groups either directly without further ordering or ranked by their coefficient values in descending order.

To keep computational complexity manageable in feature grouping, these techniques are usually applied to a smaller subset of the data. Both the MIS and EN techniques require only the data and the labels as inputs, and they operate independently of the kernel function.

\subsection{Preconditioning Additive Kernels} \label{subsec:Preconditioning}
Estimating the negative log marginal likelihood \eqref{eq:GP_approx_objective} and its derivatives \eqref{eq:div} involves repeatedly solving linear systems associated with the kernel matrix, where the data points remain fixed, but the hyperparameters vary during the optimization process. Different hyperparameters can substantially alter the properties of the kernel matrix. 
Figure~\ref{fig:varyingl} (left) illustrates the iteration counts for the unpreconditioned CG to solve $20$ linear systems with regularized additive Gaussian kernel matrices to a relative residual tolerance of $10^{-3}$ with zero initial vector.
These systems share a common random right-hand side. 
These $20$ kernel matrices are associated with the same $1000$ points in $\mathbb{R}^6$ and fixed $\sigma_f^2=\frac{1}{P}$, $\sigma_\varepsilon^2=0.01$, but different length-scales $\ell$.
The six features are divided into three two-dimensional windows, each randomly sampled within a circle of radius $\sqrt{\frac{1000}{\pi}}$.
The figure demonstrates that solving the linear system can be particularly challenging for a range of length-scales $\ell$ that are neither very large nor very small.

\begin{figure}[ht]
\centering
\includegraphics[width=0.8\linewidth]{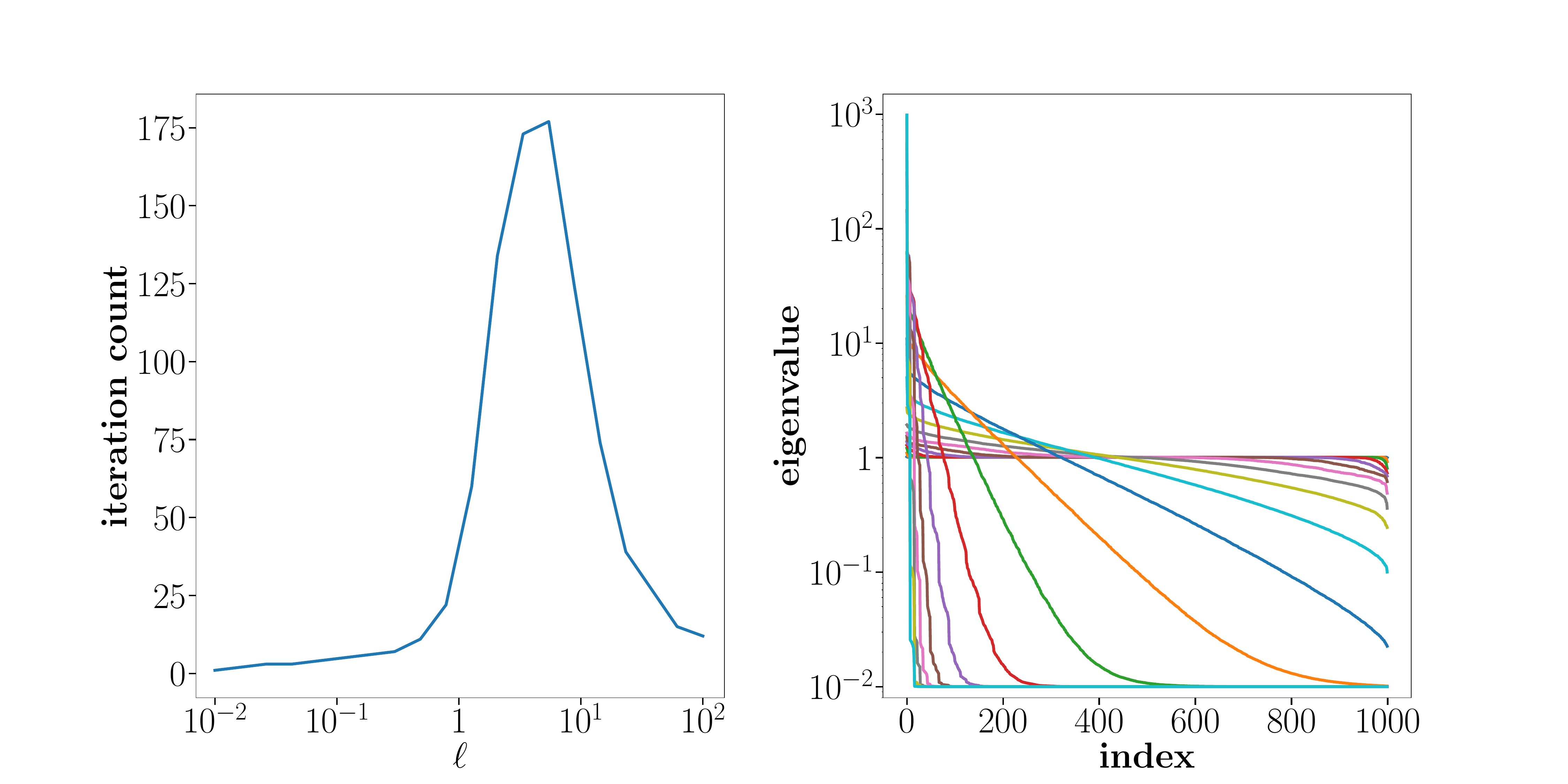}
\caption{Left: Iteration counts of unpreconditioned CG to solve linear systems for $20$ regularized additive Gaussian kernel matrices with the same random right-hand side to reach a relative residual tolerance $10^{-3}$. 
These $20$ matrices are associated with the same $1000$ points in $\mathbb{R}^6$ and fixed $\sigma_f^2=\frac{1}{P}$, $\sigma_\varepsilon^2=0.01$, but different length-scales $\ell$.
The six features are split into three two-dimensional windows.
Each window is sampled randomly within a circle of radius $\sqrt{\frac{1000}{\pi}}$.
Right: Spectra of the $20$ regularized additive Gaussian kernel matrices.}
\label{fig:varyingl}
\end{figure}

Preconditioning can be explored to enhance convergence. A detailed theoretical analysis of how various types of preconditioners can improve the asymptotic convergence rates of iterative methods used in GPs is available in \cite{wenger2022preconditioning}. Nonetheless, the selection of a practical preconditioner remains challenging, largely due to the diverse characteristics of kernel matrices caused by various hyperparameters during the optimization process \cite{zhao2023adaptive}. For instance, Figure~\ref{fig:varyingl} (right) displays the spectral properties of $20$ kernel matrices, revealing low-rank traits at higher values of $\ell$ and full rank at lower values. As a result, we choose to modify the \textit{adaptive factorized Nystr\"om} (AFN) preconditioner, which has shown to provide robust and consistent performance for non-additive kernels \cite{zhao2023adaptive}. For a dataset $\mathcal{X}$, the AFN preconditioner identifies \(k\) landmark points based on the estimated numerical rank of the kernel matrix. This process adaptively divides the matrix into a $2\times 2$ block structure, where the \((1,1)\) block corresponds to the landmark points, and the \((2,2)\) block corresponds to the remaining data. The preconditioner $M$ is then constructed by the Cholesky factorization of the \((1,1)\) block and the approximate inverse of the Schur complement. For additive kernels, we apply \textit{farthest point sampling} (FPS) to select the landmark points from each feature window and then merge the data indices of these selections to form the $(1,1)$ block. This modified version of the AFN preconditioner for additive kernels is termed \texttt{AAFN}.


\section{Fourier-Accelerated Kernel Matrix Vector Multiplication} \label{sec:Fourier_Accelerated}
When employing iterative methods to evaluate \eqref{eq:GP_approx_objective} and \eqref{eq:div}, the multiplication of dense kernel matrices with vectors emerges as the primary computational bottleneck. To address this issue, several techniques have been developed to first approximate the dense kernel matrices using low-rank or sparse matrices. Examples include the Nyström approximation \cite{martinsson2019randomized,williams2000using,anchor,anchor2,PostGP}, random Fourier features (RFF) \cite{rahimi2007random}, structured kernel interpolation (SKI) \cite{wilson2015kernel}, or hierarchical matrices \cite{borm2003introduction,ddh2}.

In contrast, this paper proposes the Non-equispaced Fast Fourier Transform (NFFT) to directly approximate the kernel matrix-vector multiplication. The NFFT-accelerated fast summation approach, especially beneficial for additive kernel structures, employs Fourier theory to offer strong theoretical guarantees and reduce setup costs \cite{nestlerlearning,wagner2023preconditioned,wagner2024fast}. Within our framework, Fourier acceleration is individually applied to each sub-kernel \(K_s\), \(s=1,\dots,P\), with the dimensionality of each \(d_s\) capped at 3 (\(d_{\text{max}}=3\)) to maintain the computational efficiency of the NFFT technique. To simplify the notation in the analysis, \(d\) specifically denotes the dimensionality \(d_s\) of each sub-kernel in the remaining sections.

\subsection{NFFT for Additive Kernels}

The NFFT method first approximates each windowed kernel function \(\kappa\) in the additive kernel by a periodically continued function \(\kappa_{\text{R}}\)
within a bounded domain. To ensure data points fit within a bounded domain, each data point \(\bm{x}_i^{\mathcal{W}_s}\) in feature window \(\mathcal{W}_s\) is scaled to fall within the interval \([-\tfrac14, \tfrac14)^d\) in this paper. These scaled data points are denoted by \(\tilde{\bm{x}}^{\mathcal{W}_s}_i\).

Define $\bm{r}$ as $\bm{x}-\bm{y}$ where $\bm{x}$ and $\bm{y}$ are the inputs to the windowed kernel function $\kappa$. Then the NFFT method first properly extends the kernel to a smooth or at least continuous periodic function $\kappa_{R}$ that is then approximated using a truncated Fourier series $\kappa_{\text{RF}}$:
\begin{align}
  \kappa(\bm{r}) = \kappa_{\text{R}} (\bm{r}) \approx \kappa_{\text{RF}} (\bm{r}) 
      = \sum_{\bm{k}\in\mathcal{I}_m} {b}_{\bm{k}}(\kappa_\text{R}) \, \mathrm e^{2\pi\mathrm{i}\bm{k}^\intercal\bm{r}},
    \label{eq:nfftkappa}  
\end{align}
where \(\mathcal{I}_m := \{\bm{k} \in \mathbb{Z}^d : -\tfrac{m}{2} \leq k_j < \tfrac{m}{2} \,\, \forall j=1,\dots, d\}\) is a multivariate index set with a cardinality of \(|\mathcal I_{m}|=m^d\) and the discrete Fourier coefficients ${b}_{\bm{k}}(\kappa_\text{R})$ are given by
\begin{align}\label{eq:discrete_fourier_coefficients1}
  {b}_{\bm k}(\kappa_\text{R}) \coloneqq \frac{1}{m^d} \sum_{\bm{l}\in\mathcal{I}_m} \kappa_\text{R}\Bigl(\frac{\bm{l}}{m}\Bigr) \, \mathrm e^{-2\pi\mathrm{i}\bm{l}^\intercal\bm{k}/m}.
\end{align}
In the simplest case, the function \(\kappa_\text{R}\) is simply the periodic continuation of \(\kappa\).
For an illustration, see Figure~\ref{fig:kernel_approx}, which depicts a one-dimensional kernel, its periodic continuation and NFFT approximation over the interval $[-\tfrac12,\tfrac12)$. It is possible to smoothen the inner or outer boundaries \cite{potts2003fast} but in our implementation, the outer boundary smoothing is set to zero.

\begin{figure}
    \centering
    \begin{tikzpicture}\footnotesize
        \begin{axis}
        [width=0.35\linewidth]
        \addplot[color=blue,solid,line width=0.5pt,domain=-0.5:0.5,samples=300]{exp(-abs(x)/0.3)};
        \node[blue] at (axis cs:0.0,0.3) {$\kappa(\cdot)$};
        \end{axis}
    \end{tikzpicture}
    \begin{tikzpicture}\footnotesize
        \begin{axis}
        [width=0.35\linewidth]
        \addplot[color=blue,solid,line width=0.5pt,domain=-0.5:0.5,samples=300]{exp(-abs(x)/0.3)};
        \addplot[color=blue,solid,line width=0.5pt,domain=-0.6:-0.5,samples=50]{exp(-abs(x+1)/0.3)};
        \addplot[color=blue,solid,line width=0.5pt,domain=0.5:0.6,samples=50]{exp(-abs(x-1)/0.3)};
        \node[blue] at (axis cs:0.0,0.3) {$\kappa_\text{R}(\cdot)$};
        \end{axis}
    \end{tikzpicture}
    \begin{tikzpicture}\footnotesize
        \begin{axis}
        [width=0.35\linewidth]
        \addplot[color=blue,domain=-0.5:0.375,samples=8,mark=*,only marks]{exp(-abs(x)/0.3)};
        \addplot[color=blue,domain=-0.6:0.6,samples=300,smooth]{0.4936953592466232 +
        0.1672815584818328*cos(360*x) + 
        0.03995989822792567*cos(2*360*x) + 
        0.03549954080877673*cos(3*360*x) + 
        0.02082264571630632*cos(4*360*x) + 
        0.03549954080877673*cos(3*360*x) + 
        0.03995989822792567*cos(2*360*x) + 
        0.1672815584818328*cos(1*360*x)};
        \node[blue] at (axis cs:0.0,0.3) {$\kappa_\text{RF}(\cdot)$};
        \end{axis}
    \end{tikzpicture}
    \caption{\label{fig:kernel_approx}
    Visualization in 1D: The original kernel function $\kappa$ (left), the periodically continued kernel function $\kappa_\textbf{R}$ (middle) and its Fourier approximation $\kappa_\text{RF}$ (right). The Fourier approximation $\kappa_\text{RF}$ is a trigonometric polynomial interpolating $m$ (here $m=8$) equidistant samples of the kernel function (dots).} 
\end{figure}
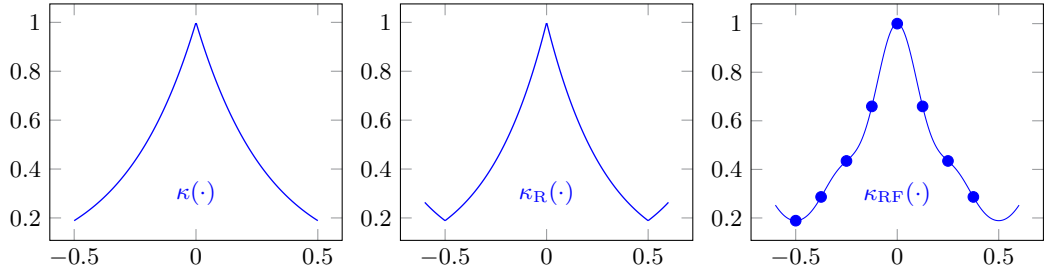

Based on \eqref{eq:nfftkappa}, NFFT finally approximates the summation 
\(h(\bm{x}^{\mathcal{W}_s}_i) \coloneqq \sum_{j=1}^n v_j \kappa(\bm{x}^{\mathcal{W}_s}_i, \bm{x}^{\mathcal{W}_s}_j)\) for all \(i=1, \dots, n\) as:
\begin{align} \label{eq:Fourier_sum}
    h(\bm{x}^{\mathcal{W}_s}_i) \approx h_\approx(\bm{x}^{\mathcal{W}_s}_i) := \sum_{\bm{k} \in \mathcal{I}_m} {b}_{\bm{k}}(\kappa_\text{R}) \left( \sum_{j=1}^{n} v_j \e^{-2 \pi \mathrm{i} \bm{k}^ \intercal \tilde{\bm{x}}^{\mathcal{W}_s}_j} \right) \e^{2 \pi \mathrm{i} \bm{k}^ \intercal \tilde{\bm{x}}^{\mathcal{W}_s}_i}.
\end{align}

By applying the adjoint NFFT for the inner sums and NFFT for the outer sums in \eqref{eq:Fourier_sum}, the summation \(h(\bm{x}^{\mathcal{W}_s}_i)\) can be approximated efficiently. This approach significantly reduces the arithmetic complexity of approximating a matrix-vector product to \(\mathcal{O}(n\log n)\). For further details on the underlying theory and the implementation of NFFT, refer to Appendix \ref{sec:nffti}.

To leverage the full computational power of the NFFT, the dimension of the feature space must be small. This is because the cost of computing Fourier coefficients increases exponentially with dimensionality. Indeed, most techniques for accelerating kernel evaluations have in common, that their effectiveness is restricted to small feature dimensions. As a result, we are splitting the feature space and work with a sum of sub-kernels \(K_s\) relying on smaller feature groups \(\mathcal{W}_s\) with \(|\mathcal{W}_s|=d_s, s=1,\dots,P\). For applying the fast summation approach to the additive kernel as introduced in~\eqref{eq:additive_kernel}, a Fourier approximation must be performed separately for each sub-kernel \(K_s\). The approximate products \(K_s\bm{v}\) are then summed up and weighted accordingly. 

\subsection{NFFT for Derivative Kernels\label{sec:nfft_for_deriv}}
We point out that the NFFT approximation of the derivative kernel exactly matches the derivative of the NFFT-approximated kernel, which is crucial for working with the correct gradients for the GP hyperparameter optimization. This consistency is validated through analysis with the Mat\'ern\texorpdfstring{$(\tfrac{1}{2})$}{} kernel below.

Consider the Mat\'ern\texorpdfstring{$(\tfrac{1}{2})$}{} kernel $\kappa^{\text{m}}(\bm{r})= \e^{-\|\bm{r}\|_2/\ell}$ and its analytical derivative with respect to \(\ell\) given by $\kappa^{\text{derm}}(\bm{r})= \tfrac{\|\bm{r}\|_2}{\ell^2} \mathrm e^{-\|\bm{r}\|_2/\ell}$.
Based on \eqref{eq:nfftkappa} and \eqref{eq:discrete_fourier_coefficients1}, we can express $\kappa_\text{RF}$ as
\begin{align*}
  \kappa_{\text{RF}} (\bm{r}) 
  &= \frac{1}{m^d}\sum_{\bm{k}\in\mathcal{I}_m} \sum_{\bm{l}\in\mathcal{I}_m}
  \kappa_{\text{R}}\!\Bigl(\tfrac{\bm{l}}{m}\Bigr) 
  \, \mathrm e^{-2\pi\mathrm{i}\bm{l}^\intercal\bm{k}/m} \, \mathrm e^{2\pi\mathrm{i}\bm{k}^\intercal\bm{r}},
\end{align*}
where $\kappa \in \{\kappa^{\text{m}}, \kappa^{\text{derm}}\}$. Differentiating \(\kappa_{\text{RF}}^{\text{m}}(\bm{r})\) with respect to \(\ell\) reveals
\begin{equation}\label{eq:exact_derivative}
\frac{\partial}{\partial \ell}\,\kappa_{\text{RF}}^{\text{m}}(\bm{r})
=\frac{1}{m^d}\sum_{\bm{k}\in\mathcal{I}_m} \sum_{\bm{l}\in\mathcal{I}_m}
  \left\{ \frac{\partial}{\partial \ell}\kappa^{\text{m}}_{\text{R}}\!\Bigl(\tfrac{\bm{l}}{m}\Bigr)\right\} 
  \, \mathrm e^{-2\pi\mathrm{i}\bm{l}^\intercal\bm{k}/m} \, \mathrm e^{2\pi\mathrm{i}\bm{k}^\intercal\bm{r}} =\kappa_{\text{RF}}^{\text{derm}}(\bm{r}),
\end{equation}
demonstrating that the derivative of the approximate kernel function \(\kappa_{\text{RF}}^{\text{m}}\) coincides exactly with the NFFT approximated derivative kernel \(\kappa_{\text{RF}}^{\text{derm}}\), so we indeed provide the correct gradients using this method. This conclusion extends analogously to other possible kernel functions and corresponding derivatives with respect to their kernel parameters. 

\section{Error Analysis} \label{sec:Theoretical_Analysis}
We presented the Fourier acceleration as a method to efficiently replace the dense kernel matrix vector multiplication with an efficient approximation of reduced complexity. In this section, we now want to provide a rigorous analysis of how close this approximation stays to the true kernel function and its derivative. As shown in \eqref{eq:Fourier_sum}, the NFFT-based fast summation method employs a truncated Fourier series to approximate kernel-vector multiplications. The error in this approximation can be quantified using the Hölder inequality:
\begin{align} \label{eq:Fourier_approx_error}
    \left| h(\bm{x}_i) - h_\approx(\bm{x}_i) \right| &= \left| \sum_{j=1}^n v_j \kappa_\text{ERR}(\bm{x}_i, \bm{x}_j) \right|
    \le \| \bm{v} \|_1 
    \max_{\bm x_j} \left|\kappa_\text{ERR}(\bm x_i,\bm x_j)\right|
    \quad \forall i = 1, \dots, n,
\end{align}
where \(\| \bm{v} \|_1 = \sum_{j=1}^n | v_j |\), and
and $\kappa_\text{ERR}=\kappa-\kappa_\text{RF}$
denotes the difference between the actual kernel and its truncated Fourier series representation.
To simplify the analysis, we assume $\bm{x}_i\in [-\tfrac{1}{4}, \tfrac{1}{4})^d$ and the maximum error is then computed over all points \(\bm{r}_{ij} = \bm{x}_i - \bm{x}_j \in [-\tfrac{1}{2}, \tfrac{1}{2})^d\), where \(\bm{r}_{ij}\) denotes the relative position vector between any two points in the domain. 

In this section, we derive upper bounds for $\kappa_\text{ERR}$ associated with Mat\'ern kernels and their derivatives to justify the efficiency of the NFFT method. The estimates provided here are directly applicable to a single sub-kernel. For an additive kernel structure composed of multiple sub-kernels, these estimates can be effectively extended by independently applying the derived bounds to each component. To the best of our knowledge, this marks the first provision of Fourier approximation error estimates for Mat\'ern kernels.

\subsection{\(1\)-Periodic Periodization}
While analytic estimates are available in one-dimensional settings, they become intractable when \(d > 1\). The primary challenge in higher dimensions stems from the interaction between the Euclidean norm $\Vert \bm{x}-\bm{y}\Vert_2$ and oscillatory terms, which excludes a straightforward decomposition of the multidimensional Fourier coefficients into products of univariate ones as done in the error analysis \cite{wagner2024fast} for the Gaussian kernel. Nevertheless, this complexity can be mitigated by adopting a 
\(1\)-periodic periodization. 

\begin{definition}[\(1\)-Periodic Periodization]
For any integrable function $f \in L_1(\mathbb{R}^d)$, the \(1\)-periodic periodization $\tilde{f}$ of $f$ is defined as
\begin{align} \label{eq:periodic_summation}
    \tilde{f} (\bm{r}) \coloneqq \sum_{\bm{l} \in \mathbb{Z}^d} f(\bm{r} + \bm{l}).
\end{align}
\end{definition}
This $\tilde{f}$ is $1$-periodic in each dimension and can serve as an accurate approximation of $f$ when $f$ retains negligible values outside the hypercube $[-\tfrac{1}{2}, \tfrac{1}{2})^d$.  Refer to Figure \ref{fig:1periodization} for a Mat\'ern\texorpdfstring{\((\frac{1}{2})\)}{} kernel example with a small length-scale $\ell$. It is also important to note that the 1-periodic periodization $\tilde{f}$ is specifically employed for analyzing Fourier approximation errors within this section. In contrast, the periodic continuation $\kappa_{\text{R}}$ as introduced in Section~\ref{sec:nfft_for_deriv} is utilized in the actual NFFT implementation. 

Let $\hat{f}(\bm{\omega}) = \int_{\mathbb{R}^d} f(\bm{x}) \e^{-2\pi\mathrm{i}\bm{\omega}^\intercal\bm{x}} \, \mathrm{d} \bm{x}$
represent the $d$-dimensional Fourier transform of an integrable function $f$. According to the Poisson summation formula \cite{plonka2018numerical}, the Fourier coefficient $c_{\bm{k}}(\tilde{f})$ corresponds to the continuous Fourier transform of $f$ as follows:
\begin{equation}
    c_{\bm{k}}(\tilde f)
    \coloneqq \int_{[-\frac12,\frac12)^d} \tilde f(\bm x) \,\e^{-2\pi\i\bm k^\intercal\bm x} \dbmx
    = \hat{f}(\bm{k}).
\end{equation}

Thus, provided that $f(\bm{r}) \approx \tilde{f} (\bm{r})$ for all $\bm{r} \in [-\tfrac{1}{2},\tfrac{1}{2})^d$:
\begin{equation}\label{eq:matern3d_periodization}
    f(\bm r) \approx \tilde f(\bm r)
    = \sum_{\bm{k}\in\mathbb{Z}^d} c_{\bm{k}}(\tilde{f}) \,\e^{2\pi\mathrm{i}\bm{k}^\intercal\bm{r}}
    = \sum_{\bm{k}\in\mathbb{Z}^d} \hat f(\bm k) \,\e^{2\pi\mathrm{i}\bm{k}^\intercal\bm{r}},
\end{equation}
we have
\begin{equation}
\label{eq:approxc}
c_{\bm k}(f)\approx c_{\bm k}(\tilde f)=\hat f(\bm k).
\end{equation}

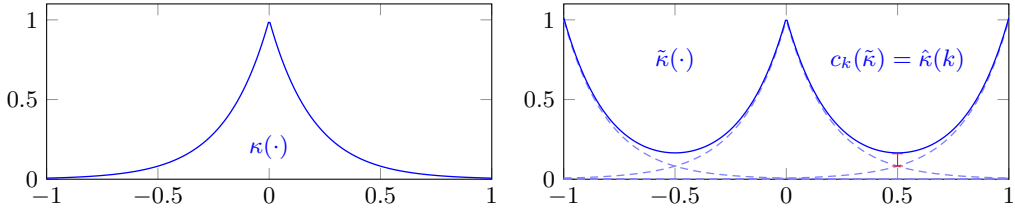
\begin{figure}
    \centering
    \begin{tikzpicture}\footnotesize
        \begin{axis}
        [width=0.48\linewidth,ymin=0,ymax=1.1,xmin=-1,xmax=1,height=0.25\linewidth]
        \addplot[color=blue,solid,line width=0.5pt,domain=-1.0:1.0,samples=300]{exp(-abs(x)/0.2)};
        \node[blue] at (axis cs:0.0,0.2) {$\kappa(\cdot)$};
        \end{axis}
    \end{tikzpicture}
    \begin{tikzpicture}\footnotesize
        \begin{axis}
        [width=0.48\linewidth,ymin=0,ymax=1.1,xmin=-1,xmax=1,height=0.25\linewidth]
        \addplot[color=blue!50,densely dashed,line width=0.5pt,domain=-1.0:1.0,samples=300]{exp(-abs(x)/0.2)};
        \addplot[color=blue!50,densely dashed,line width=0.5pt,domain=-1.0:1.0,samples=300]{exp(-abs(x-1)/0.2)};
        \addplot[color=blue!50,densely dashed,line width=0.5pt,domain=-1.0:1.0,samples=300]{exp(-abs(x+1)/0.2)};
        \addplot[color=blue!50,densely dashed,line width=0.5pt,domain=-1.0:1.0,samples=300]{exp(-abs(x-2)/0.2)};
        \addplot[color=blue!50,densely dashed,line width=0.5pt,domain=-1.0:1.0,samples=300]{exp(-abs(x+2)/0.2)};
        \addplot[color=blue,solid,line width=0.5pt,domain=-1.0:1.0,samples=300]{exp(-abs(x)/0.2)+exp(-abs(x+1)/0.2)+exp(-abs(x-1)/0.2)+exp(-abs(x-2)/0.2)+exp(-abs(x+2)/0.2)};
        \node[blue] at (axis cs:-0.5,0.75) {$\tilde\kappa(\cdot)$};
        \node[blue] at (axis cs:0.5,0.75) {$c_k(\tilde\kappa)=\hat\kappa(k)$};
        \addplot[color=red,mark=-] coordinates{(0.5,0.083) (0.5,0.16)};
        \end{axis}
    \end{tikzpicture}
    \caption{Visualization in 1D: The original kernel function $\kappa(r)=\e^{-|r|/\ell}$ with $\ell=0.2$ (left) and its 1-periodic periodization $\tilde\kappa$ (right).
    The Fourier coefficients of the periodization are given by the Fourier transform of $\kappa$, evaluated at integer values. The difference between $\kappa$ and $\tilde{\kappa}$ is small when $\ell$ remains small.  \label{fig:1periodization}}
\end{figure}

When we only use finitely many Fourier coefficients with frequencies $\bm{k}\in\mathcal{I}_m$, we can estimate the Fourier approximation error in \eqref{eq:Fourier_approx_error} as
\begin{align} \label{eq:error_estimate_mat3d}
    \|\kappa_{\text{ERR}}\|_{\infty} = \max_{\bm{r} \in [-\frac12, \frac12)^d} |\kappa_{\text{ERR}}(\bm r)| \leq 2 \sum_{\bm{k}\in \mathbb{Z}^d\setminus\mathcal{I}_m} |c_{\bm{k}}(\kappa)|
    \approx 2 \sum_{\bm{k}\in \mathbb{Z}^d\setminus\mathcal{I}_m} |\hat \kappa(\bm k)|,
\end{align}
where the first inequality arises from the well-known aliasing formula $b_{\bm k}(\kappa)=\sum_{\bm l\in\mathbb Z^d}c_{\bm k+\bm l m}(\kappa)$, see \cite{potts2003fast, plonka2018numerical} and references therein.
In the following two lemmas, we show that the error between the trivariate Mat\'ern\texorpdfstring{\((\frac{1}{2})\)}{} kernel $\kappa^\text{m}$ (Lemma \ref{lemma1}), its derivative $\kappa^\text{derm}$ (Lemma \ref{lemma2}), and their 1-periodic periodizations diminishes exponentially as $\ell$ approaches zero.
These analyses justify the use of the periodized functions in deriving the NFFT error estimates when $\ell$ is small.

\begin{lemma}\label{lemma1}
For $\bm r\in[-\frac12,\frac12)^3$ and the trivariate Mat\'ern\texorpdfstring{$(\tfrac{1}{2})$}{} kernel $\kappa^\mathrm{m}$, we have
\begin{equation}\label{eq:error_tilde_kappa_m}
    \max_{\bm r\in[-\frac12,\frac12)^3}\|\kappa^\mathrm{m}-\tilde \kappa^\mathrm{m}\|_{\infty}
    \leq \delta^\mathrm{m}(\ell),
\end{equation}
where the quantity $\delta^\mathrm{m}(\ell)$ is given by
\begin{equation*}
    \delta^\mathrm{m}(\ell) = 3\,\e^{-1/(2\sqrt3\ell)}(1+2\sqrt3\ell) + 3\,\e^{-1/\sqrt3\ell}(1+2\sqrt3\ell)^2 + \e^{-3/(2\sqrt3\ell)}(1+2\sqrt3\ell)^3,
\end{equation*}
which becomes negligibly small for small $\ell$.
\end{lemma}
\begin{proof} See Appendix \ref{lemma1proof}.
\end{proof}

\begin{lemma}\label{lemma2}
For $\bm r\in[-\tfrac12,\tfrac12)^3$ and the trivariate derivative Mat\'ern\texorpdfstring{$(\tfrac{1}{2})$}{} kernel $\kappa^\mathrm{derm}$ with $\ell<\frac12$ we have
\begin{equation}
    \max_{\bm r\in[-\frac12,\frac12)^3}\|\kappa^\mathrm{derm}-\tilde \kappa^\mathrm{derm}\|_{\infty}
    \leq \delta^\mathrm{derm}(\ell),
\end{equation}
where the quantity $\delta^\mathrm{derm}(\ell)$ is given by
\begin{equation*}
    \delta^\mathrm{derm}(\ell) =
    \frac{3}{\ell^2}\left(1+\e^{-1/(2\sqrt3\ell)}(1+2\sqrt3\ell)\right)^2
    \left(1+\e^{-1/(2\sqrt3\ell)}(1+2\sqrt3\ell+12\ell^2)\right)
    -\frac3{\ell^2}
\end{equation*}
which becomes negligibly small for small $\ell$.
\end{lemma}

\begin{proof}
See Appendix \ref{lemma2proof}.
\end{proof}

\subsection{Fourier Approximation Error for Periodized Kernels}
In this section, we analyze the Fourier approximation error \eqref{eq:error_estimate_mat3d} using the 1-Periodic Periodized Mat\'ern kernel functions. This analysis provides insight into how the error decays for the true kernels at small length-scales.

We first derive the error estimate for the periodized trivariate Mat\'ern\texorpdfstring{\((\frac{1}{2})\)}{} kernel in the next theorem.
\begin{theorem}[Fourier Error Estimate for the Periodized Trivariate Mat\'ern\texorpdfstring{\((\frac{1}{2})\)}{} Kernel] \label{theorem:Fourier_Error_Matern_3d}
For the periodized trivariate Mat\'ern\texorpdfstring{\((\frac{1}{2})\)}{} kernel $\tilde\kappa^\mathrm{m}$ on $[-\frac12,\frac12)^3$, we have
\begin{align*}
\|\tilde\kappa^\mathrm{m}_{\mathrm{ERR}}\|_{\infty}
\leq
\frac{8}{\pi^2\ell(m-2\sqrt{3})}.
\end{align*}
\end{theorem}

\begin{proof}
By \cite{williams2006gaussian}, the $d$-dimensional Fourier transform of a function $g(\bm{x})=\e^{-2\pi\alpha\|\bm{x}\|_2}$ is given as
\begin{align*}
    \hat{g}(\bm{\omega}) = \frac{\Gamma\left(\tfrac{d+1}{2}\right)}{\pi^{(d+1)/2}} \cdot \frac{\alpha}{\left(\alpha^2 + \|\bm{\omega}\|_2^2\right)^{(d+1)/2}}.
\end{align*}
For $\alpha=\tfrac{1}{2\pi\ell}$, $g(\bm{x})$ equals the Mat\'ern\texorpdfstring{\((\frac{1}{2})\)}{} kernel $\kappa^\text{m}$, and thus
\begin{align*}
    \hat\kappa^\text{m}(\bm{\omega}) = \frac{\Gamma\left(\tfrac{d+1}{2}\right)}{\pi^{(d+1)/2}} \cdot \frac{1}{2\pi\ell} \cdot \frac{1}{\left(\tfrac{1}{4\pi^2\ell^2} + \|\bm{\omega}\|_2^2\right)^{(d+1)/2}}.
\end{align*}
For the trivariate case, we have $d=3$ and with $\Gamma(2)=1$ for the Gamma function, we obtain
\begin{align} \label{eq:Fourier_transform_matern_3d}
\hat\kappa^\text{m}(\bm{\omega}) = \frac{1}{\pi^2} \cdot \frac{1}{2\pi\ell} \cdot \frac{1}{\left(\tfrac{1}{4\pi^2\ell^2} + \|\bm{\omega}\|_2^2\right)^{2}} \leq \frac{1}{2\pi^3\ell\|\bm{\omega}\|_2^4},
\end{align}
where the inequality is tighter if $\tfrac{1}{4\pi^2\ell^2}$ is small compared to $\|\bm{\omega}\|_2^2$.

By \eqref{eq:Fourier_transform_matern_3d} and by estimating the sums via adding more summands to the error sum, 
we now estimate
\begin{align*}
    2 \sum_{\bm{k}\in\mathbb{Z}^3\setminus\mathcal{I}_m} |c_{\bm{k}}(\tilde\kappa^{m})| \leq \frac{2}{2\pi^3\ell} \sum_{\bm{k}\in\mathbb{Z}^3\setminus\mathcal{I}_m} \frac{1}{\|\bm{k}\|_2^4} \leq \frac{1}{\pi^3\ell} \sum_{\substack{\bm{k}\in\mathbb{Z}^3 \\ \|\bm{k}\|_2\geq m/2}} \frac{1}{\|\bm{k}\|_2^4}
    \leq \frac{1}{\pi^3\ell} \int_{\substack{\bm{x}\in\mathbb{R}^3 \\ \|\bm{x}\|_2 \geq m/2-\sqrt{3}}} \frac{1}{\|\bm{x}\|_2^4} \, \mathrm{d}\bm{x},
\end{align*}
where the last inequality is obtained by estimating the sum from above, making use of the fact that $\tfrac{1}{\bm{x}^4}$ is monotonically decreasing and taking the shift from integer to real variables into account.
For computing the last integral, we employ polar coordinates in three dimensions, where $r\in[m/2-\sqrt{3},\infty)$, $\varphi \in [0,2\pi)$, $\vartheta \in [0,\pi)$ and $r^2 \sin(\vartheta)$ is the Jacobian determinant.
With this, we compute
\begin{align*}
    \int_{\substack{\bm{x}\in\mathbb{R}^3 \\ \|\bm{x}\|_2 \geq m/2-\sqrt{3}}} \frac{1}{\|\bm{x}\|_2^4} \, \mathrm{d}\bm{x} &= \int_{m/2-\sqrt{3}}^\infty \int_{0}^{2\pi} \int_{0}^\pi \frac{r^2 \sin(\vartheta)}{r^4} \, \mathrm{d}\vartheta\, \mathrm{d} \varphi \,\mathrm{d}r
    = \frac{8\pi}{m-2\sqrt{3}}.
\end{align*}
In summary, this gives
\begin{align*}
    \max_{\bm r\in[-\frac12,\frac12)^3} \left|\tilde\kappa^\text{m}(\bm r)-\tilde\kappa^\text{m}_{\text{RF}}(\bm r)\right| \leq
 	2 \sum_{\bm{k}\in\mathbb{Z}^3\setminus\mathcal{I}_m} |c_{\bm{k}}(\tilde\kappa^\text{m})| \leq \frac{1}{\pi^3\ell} \cdot \frac{8\pi}{m-2\sqrt{3}} = \frac{8}{\pi^2\ell(m-2\sqrt{3})},
\end{align*}
where we used \eqref{eq:error_estimate_mat3d}.
\end{proof}

In the proof of Theorem \ref{theorem:Fourier_Error_Matern_3d}, the inequality~\eqref{eq:Fourier_transform_matern_3d} indicates that the estimate becomes tighter as the term $\frac{1}{4\pi^2\ell^2}$ becomes negligible compared to $\|\bm\omega\|_2^2$.  Since $\|\bm\omega\|_2 = \|\bm k\|_2 \geq \frac{m}{2}$,  this aligns with the condition $\ell\pi m > 1$. The condition $\ell\pi m > 1$ offers valuable guidance for selecting a suitable expansion degree \(m\) based on the length-scale \(\ell\), although it is not mandatory for the validity of the estimate. In \cite{wagner2024fast}, analogous conditions were crucial for the error estimates of the Gaussian kernel, highlighting their importance in kernel error analysis.
In the next theorem, we analyze the Fourier approximation error for the derivative Mat\'ern\texorpdfstring{\((\frac{1}{2})\)}{} kernel in a very similar fashion.

\begin{theorem}[Fourier Error Estimate for the Periodized Trivariate Derivative Mat\'ern\texorpdfstring{\((\frac{1}{2})\)}{} Kernel] \label{theorem:Fourier_Error_derMatern_3d}
For the trivariate derivative Mat\'ern\texorpdfstring{\((\frac{1}{2})\)}{} kernel $\tilde\kappa^\mathrm{derm}$ on $[-\frac12,\frac12)^3$, we have
\begin{align*}
\|\tilde\kappa^\mathrm{derm}_{\mathrm{ERR}}\|_{\infty}
&\leq \frac{32}{\ell^4\pi^4 3(m-2\sqrt{3})^3} + \frac{8}{\ell^2\pi^2(m-2\sqrt{3})}.
\end{align*}
\end{theorem}

\begin{proof}
Let us have a look at the Mat\' ern\texorpdfstring{$(\frac{1}{2})$}{} kernel $\kappa^{\text{m}}(\bm{x}) = \e^{-\|\bm{x}\|_2/\ell}$ and the Mat\'ern\texorpdfstring{$(\frac{3}{2})$}{} kernel $\kappa^{\text{m3}}(\bm{x}) \coloneqq \left( 1 + \tfrac{\sqrt{3}\|\bm{x}\|_2}{l} \right) \e^{-\sqrt{3}\|\bm{x}\|_2/l}$, with $l=\sqrt{3}\ell$. 
Then, the trivariate derivative Mat\'ern\texorpdfstring{$(\frac{1}{2})$}{} kernel $\kappa^\text{derm}(\bm x)=\tfrac{\|\bm{x}\|_2}{\ell^2}\e^{-\|\bm{x}\|_2/\ell}$ can be expressed as $\kappa^\text{derm}(\bm{x}) = \tfrac{1}{\ell} \left[ \kappa^{\text{m3}}(\bm{x}) - \kappa^{\text{m}}(\bm{x}) \right]$. By \cite{williams2006gaussian}, the Fourier transform of the Mat\'ern class of functions is given by
\begin{align}\label{eq:spectral_matern}
	S(\bm{\omega}) = \frac{2^d \pi^{d/2} \Gamma(\nu+ d/2)(2\nu)^{\nu}}{\Gamma(\nu)\ell^{2\nu}} \left( \tfrac{2\nu}{\ell^2} + 4\pi^2 \|\bm{\omega}\|_2^2\right)^{-(\nu+d/2)}.
\end{align}
For $\nu=3/2$ and $d=3$, we obtain the Fourier transform of the Mat\'ern\texorpdfstring{$(\frac{3}{2})$}{} kernel
\begin{align*}
    \hat\kappa^{\text{m3}} (\bm{\omega}) &= \frac{2^3 \pi^{3/2} \Gamma(3) 3^{3/2}}{\Gamma\left(\tfrac{3}{2}\right) l^3} \left( \tfrac{3}{l^2} + 4\pi^2\|\bm{\omega}\|_2^2 \right)^{-3}
    = \frac{2^5 \pi 3^{3/2}}{l^3} \frac{1}{\left(\tfrac{3}{l^2} + 4\pi^2\|\bm{\omega}\|_2^2 \right)^3} \\
    &= \frac{3^{3/2}}{2l^3 \pi^5} \frac{1}{\left(\tfrac{3}{4\pi^2l^2} + \|\bm{\omega}\|_2^2 \right)^3}
    = \frac{1}{2\ell^3 \pi^5} \frac{1}{\left(\tfrac{1}{4\pi^2\ell^2} + \|\bm{\omega}\|_2^2 \right)^3}
    \leq \frac{1}{2\ell^3 \pi^5\|\bm{\omega}\|_2^6}, 
\end{align*}
with $\Gamma(3)=2$, $\Gamma\left(\tfrac{3}{2}\right) = \tfrac{1}{2} \sqrt{\pi}$ and inserting $l=\sqrt{3}\ell$, where the inequality is tighter if $0 < \tfrac{1}{4\pi^2\ell^2}$ is small compared to $\|\bm{\omega}\|_2^2$.

By \eqref{eq:error_estimate_mat3d} and analogously to the proof of Theorem~\ref{theorem:Fourier_Error_Matern_3d}, we can now estimate the Fourier approximation error of the periodized kernel $\tilde\kappa^\text{m3}$ as
\begin{align*}
    2 \sum_{\bm{k}\in\mathbb{Z}^3\setminus\mathcal{I}_m} |c_{\bm{k}}\left(\tilde\kappa^{\text{m3}}\right)| &\leq \frac{2}{2\ell^3\pi^5} \sum_{\bm{k}\in\mathbb{Z}^3\setminus\mathcal{I}_m} \frac{1}{\|\bm{k}\|_2^6}
    \leq \frac{1}{\ell^3\pi^5} \int_{\substack{\bm{x}\in\mathbb{R}^3 \\ \|\bm{x}\|_2 \geq m/2-\sqrt{3}}} \frac{1}{\|\bm{x}\|_2^6} \, \mathrm{d}\bm{x}.
\end{align*}
Introducing polar coordinates in three dimensions, we compute the last integral as
\begin{align*}
	\int_{\substack{\bm{x}\in\mathbb{R}^3 \\ \|\bm{x}\|_2 \geq m/2-\sqrt{3}}} \frac{1}{\|\bm{x}\|_2^6} \, \mathrm{d}\bm{x} &= \int_{m/2-\sqrt{3}}^\infty \int_{0}^{2\pi} \int_{0}^\pi \frac{r^2 \sin(\vartheta)}{r^6} \, \mathrm{d}\vartheta \,\mathrm{d} \varphi \,\mathrm{d}r
	= \frac{2^5\pi}{3(m-2\sqrt{3})^3}.
\end{align*}
Putting everything together gives
\begin{align*}
	\|\tilde\kappa^\text{m3}_{\text{ERR}}\|_{\infty} \leq 2 \sum_{\bm{k}\in\mathbb{Z}^3\setminus\mathcal{I}_m} |c_{\bm{k}}(\tilde\kappa^{\text{m3}})|
	\leq \frac{1}{\ell^3\pi^5} \frac{2^5\pi}{3(m-2\sqrt{3})^3}
	= \frac{2^5}{\ell^3\pi^4 3(m-2\sqrt{3})^3}
\end{align*}
for the Mat\'ern\texorpdfstring{$(\frac{3}{2})$}{} kernel. By the triangle inequality,
\begin{align*}
|\hat{\kappa}^\text{derm}(\bm{\omega})| = \tfrac{1}{\ell} \left| \hat\kappa^{\text{m3}}(\bm{\omega}) - \hat\kappa^{\text{m}}(\bm{\omega}) \right| \leq \tfrac{1}{\ell} \left( \left| \hat\kappa^{\text{m3}}(\bm{\omega}) \right| + \left| \hat\kappa^{\text{m}}(\bm{\omega})\right| \right),
\end{align*}
that is, with the error estimate from Theorem~\ref{theorem:Fourier_Error_derMatern_3d}
\begin{align*}
	\|\tilde\kappa^\text{derm}_{\text{ERR}}\|_{\infty}
	&\leq \frac{32}{\ell^4\pi^4 3(m-2\sqrt{3})^3} + \frac{8}{\ell^2\pi^2(m-2\sqrt{3})}.
\end{align*}
\end{proof}

Similar to the discussion after Theorem~\ref{theorem:Fourier_Error_Matern_3d}, the presented error bound for the periodized derivative kernel is tight if \(\ell\pi m>1\).

\subsection{Total Fourier Approximation Error}
Lemmas \ref{lemma1}--\ref{lemma2} quantify the error between the kernels $\kappa\in\{\kappa^\text{m},\kappa^\text{derm}\}$ and their corresponding periodizations $\tilde\kappa$, while Theorems \ref{theorem:Fourier_Error_Matern_3d}--\ref{theorem:Fourier_Error_derMatern_3d} provide the error estimate between these periodized kernels $\tilde\kappa$ and their Fourier approximations $\tilde\kappa_\text{RF}$. Thus, we can analyze the error between the original kernel $\kappa$ and its Fourier approximation $\kappa_\text{RF}$ as follows:
\begin{align}
    \left|\kappa-\kappa_\text{RF}\right|
    \leq& \left|\kappa-\tilde\kappa\right| + \left| \tilde\kappa - \tilde\kappa_\text{RF} \right|+\left|\kappa_{\text{RF}}-\tilde\kappa_{\text{RF}}\right|,
    \label{eq:3sums}
\end{align}
where the last term on the right-hand side of the above inequality is the difference between the two Fourier approximations and can be expanded as
\begin{equation}\label{eq:breve_kappa_RF_aliasing}
\kappa_\text{RF}(\bm r)-\tilde\kappa_\text{RF}(\bm r)
= \sum_{\bm k\in\mathcal I_m} b_{\bm k}(\kappa-\tilde\kappa)\,\e^{2\pi\mathrm i\bm k^\intercal\bm r}
= \sum_{\bm k\in\mathcal I_m} \left(\sum_{\bm l\in\mathbb Z^3} c_{\bm k+\bm l m}(\kappa-\tilde\kappa)\right)\,\e^{2\pi\mathrm i\bm k^\intercal\bm r},
\end{equation}
where we make use of the aliasing formula.

Similar to $|\kappa-\tilde\kappa|$, $|\kappa_\text{RF}-\tilde\kappa_\text{RF}|$ will be negligibly small for small $\ell$. This is because  the function $\kappa-\tilde\kappa$ is differentiable on $[-\frac12,\frac12)^3$ and its analytic Fourier coefficients $c_{\bm k}(\kappa-\tilde\kappa)$ tend to zero with order $(k_1k_2k_3)^{-2}$. 
This can be achieved by partial integration two times with respect to each dimension. In 1D this gives for $k\neq0$
\begin{align*}
    c_k(\breve\kappa)=
    \int_{-\frac12}^{\frac12} \breve\kappa(x)\e^{2\pi\i kx}\dx
    &=\frac{1}{(2\pi\i k)^2}\left(
    \left[ \breve\kappa'(x)\e^{2\pi\i kx} \right]_{-\frac12}^{\frac12}
    +\int_{-\frac12}^{\frac12} \breve\kappa''(x)\e^{2\pi\i kx}\dx
    \right)\\
    &=-\frac{1}{4\pi^2k^2}\left(
    2(-1)^k\breve\kappa'(\tfrac12) +\int_{-\frac12}^{\frac12} \breve\kappa''(x)\e^{2\pi\i kx}\dx,
    \right)
\end{align*}
where we use the shorthand notation \(\breve\kappa\coloneqq\kappa-\tilde\kappa\) and 
\[
\breve\kappa'(x)=\frac{1}{\ell}\sum_{n\neq 0} \mathrm{sign}(x+n)\,\e^{-|x+n|/\ell}
\quad\text{ and }\quad 
\breve\kappa'(x)=\frac{1}{\ell^2}\sum_{n\neq 0} \e^{-|x+n|/\ell}
\]
for \(\kappa=\kappa^\text{m}\), for instance, and all \(x\in[-\frac12,\frac12)\).
Using \eqref{eq:error_f_ftilde_1d}, we see 
\[
|c_k(\breve\kappa^\text{m})| \leq \frac{1}{4\pi^2k^2}\left(
\frac{2(1+2\ell)\e^{-1/(2\ell)}}{\ell} + \frac{(1+2\ell)\e^{-1/(2\ell)}}{\ell^2}
\right)
=\frac{C(\ell)}{k^2},
\]
where \(k\neq0\) and the constant \(C(\ell)\) becomes small with \(\ell\to0\).
Also note that for \(k=0\), we obtain with \eqref{eq:error_f_ftilde_1d}
\[
|c_0(\breve\kappa^\text{m})|=\left|\int_{-\frac12}^{\frac 12}\breve\kappa^\text{m}(x)\,\e^{-2\pi\i kx}\dx\right|
\leq 1\cdot\max_{x\in\left[-\frac12,\frac12\right)} |\breve \kappa^\text{m}(x)|
\leq \e^{-1/(2\ell)}(1+2\ell),
\]
which also becomes small with \(\ell\to0\).
Finally, using \eqref{eq:breve_kappa_RF_aliasing}, we see 
\begin{equation}\label{eq:estimate_breve_kappa_RF}
    |\kappa^\text{m}_\text{RF}(r)-\tilde\kappa^\text{m}_\text{RF}(r)|
    \leq \sum_{k\in\mathcal I_m} |b_k(\breve\kappa_\text{RF}^\text{m})|\leq \sum_{k\in\mathbb Z} |c_k(\breve\kappa^\text{m})|\leq \e^{-1/(2\ell)}(1+2\ell)+C(\ell)\sum_{k\neq0} \frac{1}{k^2},
\end{equation}
which is obviously finite and approaches zero in \(o(\ell^{-2}\,\e^{-1/\ell})\) as \(\ell\to0\), where for small \(\ell\) the dominating term \(\simeq \ell^{-2}\e^{-1/\ell}\) originates from the constant \(C(\ell)\) defined above. 

In three dimensions, the computations become tedious and more technical.
We obtain that \(|c_{\bm k}(\kappa_\text{ERR})|\leq C_\text{3D}(\ell)(k_1k_2k_3)^{-2}\) for \(k_1\neq0\), \(k_2\neq0\) and \(k_3\neq 0\), which is achieved by integrating partially two times with respect to each dimension in order to transform the triple integral.
For \(\bm k=\bm 0\), we can estimate \(c_{\bm 0}(\breve\kappa)\) using Lemma~\ref{lemma1} or Lemma~\ref{lemma2}.
If some of the entries in \(\bm k\) are zero and some not, we will have to combine partial integration two times with respect to all dimensions \(j\) with \(k_j\neq0\) and estimate the maximum values of the computed partial derivatives in the integral in order derive an upper bound of the corresponding Fourier coefficients.
Analogously to \eqref{eq:estimate_breve_kappa_RF}, we end up with 
\[
|\breve\kappa^\text{m}_\text{RF}(\bm r)|
\leq |c_{\bm 0}(\breve\kappa^\text{m})|+\sum_{k_1\in\mathbb Z\setminus\{0\}} \frac{C_{1D}(\ell)}{k_1^2} + \ldots + \sum_{(k_1,k_2)\in\mathbb Z^2 \atop k_1\neq0\neq k_2} \frac{C_{2D}(\ell)}{k_1^2k_2^2}+\ldots + \sum_{\bm k\in\mathbb Z^3\atop k_j\neq0, j=1,2,3} \frac{C_{3D}(\ell)}{k_1^2k_2^2k_3^2},
\]
where all the constants, including $c_{\bm 0}(\breve\kappa^\text{m})$ approach zero as $\ell\to0$ with $o(\ell^{-6}\e^{-1/\ell})$, and all the series have finite values.
This fully justifies that \(|\kappa^\text{m}-\kappa^\text{m}_\text{RF}|\approx |\tilde\kappa^\text{m}-\tilde\kappa^\text{m}_\text{RF}|\) for small \(\ell\).
The same argumentation can be followed for the derivative kernel \(\kappa^\text{derm}\).

\begin{remark}
\label{rem::bounds}
For large values of \(\ell\), the dominant terms on the right-hand side of equation \eqref{eq:3sums} are \(\left|\kappa - \tilde{\kappa} \right|\) and \(\left|\kappa_{\text{RF}}-\tilde{\kappa}_{\text{RF}}\right|\).
Thus, \eqref{eq:3sums} can no longer be used to justify the use of $\Vert\tilde{\kappa}-\tilde{\kappa}_{\text{RF}}\Vert_\infty$ as the error estimate when $\ell$ is large.
However, in this case, the trivariate Mat\'ern kernel \(\kappa^\text{m}\) approximates a constant function with value \(1\) across the entire cube \([-\frac12,\frac12)^3\), while the periodized function \(\tilde{\kappa}^\text{m}\) approaches \(8\pi\ell^3\).
As a result, the difference \(\kappa-\tilde{\kappa}\) is approximately \(1 - 8\pi\ell^3\).
When the difference between two functions is almost constant, all other Fourier coefficients will align closely, except the zeroth coefficient. For the trivariate Mat\'ern kernel \(\kappa^\text{m}\), it can be shown that $|\kappa^\text{m}-\tilde{\kappa}^\text{m}-1+8\pi\ell^3|\simeq \frac1\ell$. Therefore, the Fourier approximation error for large \(\ell\) can be estimated by the following new inequality:
\[
|\kappa^\text{m}-\kappa^\text{m}_{\text{RF}}|\leq
\underbrace{
|\kappa^\text{m}-\tilde{\kappa}^\text{m}-1+8\pi\ell^3|}_{\text{small for large $\ell$}}
+
|\tilde{\kappa}^\text{m}-\tilde{\kappa}^\text{m}_{\text{RF}}|
+
\underbrace{
|\kappa^\text{m}_{\text{RF}}-\tilde{\kappa}^\text{m}_{\text{RF}}-1+8\pi\ell^3|}_{\text{small for large $\ell$}},
\]
where the last term can be analyzed analogously to the term \(\kappa^\text{m}_{\text{RF}}-\tilde{\kappa}^\text{m}_{\text{RF}}\), following the approach used for small $\ell$.
A similar argument also holds for the derivative kernels. This shows $\Vert\tilde{\kappa}-\tilde{\kappa}_{\text{RF}}\Vert_\infty$ can still be used as an error estimate for large $\ell$.
\end{remark}

To validate the accuracy of our derived error estimates for $\left\Vert\tilde\kappa - \tilde\kappa_\text{RF} \right\Vert_{\infty}$, we numerically compare them with the actual errors $\left\Vert\kappa-\kappa_\text{RF}\right\Vert_{\infty}$ in the Fourier approximations. For the trivariate Mat\'ern\texorpdfstring{\((\frac{1}{2})\)}{} kernel and its derivative, we generate $n=10^4$ uniformly distributed random points $\bm{x}_{i} \in [-\tfrac{1}{4},\tfrac{1}{4})^3$ and evaluate the Mat\'ern\texorpdfstring{\((\frac{1}{2})\)}{} kernel $\kappa^{\text{m}}(\bm{r}_{ij}) = \e^{-\|\bm{r}_{ij}\|_2/\ell}$ and its derivative kernel $\kappa^{\text{derm}}(\bm{r}_{ij}) = \tfrac{\|\bm{r}_{ij}\|_2}{\ell^2} \e^{-\|\bm{r}_{ij}\|_2/\ell}$ at $\bm{r}_{ij}=\bm{x}_i-\bm{x}_j$, $i,j=1, \dots, n$, for different values of $\ell$. For the approximation, we compute the discrete Fourier coefficients on a regular grid of $m^3$ points in $[-\tfrac{1}{2},\tfrac{1}{2})^3$, where we choose $m \in \{16,32,64\}$.

\begin{figure}[!ht]
	\centering
	\begin{tikzpicture}
	\begin{axis}[width=0.32\linewidth,
	ymin=1e-5,ymax=1e+2,
	ymode=log,xmode=log,
	legend style={legend pos=south east, nodes=right},
	ylabel={\(\|\kappa_\text{ERR}^{\text{m}}\|_\infty\)},
	title={\(m=16\)}
	]
	\addplot[smooth,color=blue,line width=0.5pt,densely dashed] table[x=ell,y=matM16] {estimates_mat3d.txt};
	\addplot[smooth,color=blue,line width=0.5pt] table[x=ell,y=matM16] {measured_errors_mat3d.txt};
	\end{axis}
	\end{tikzpicture}
	\begin{tikzpicture}
	\begin{axis}[width=0.32\linewidth,
	ymin=1e-5,ymax=1e+2,
	ymode=log,xmode=log,
	legend style={legend pos=south east, nodes=right},
	title={\(m=32\)}
	]
	\addplot[smooth,color=blue,line width=0.5pt,densely dashed] table[x=ell,y=matM32] {estimates_mat3d.txt};
	\addplot[smooth,color=blue,line width=0.5pt] table[x=ell,y=matM32] {measured_errors_mat3d.txt};
	\end{axis}
	\end{tikzpicture}
	\begin{tikzpicture}
	\begin{axis}[width=0.32\linewidth,
	ymin=1e-5,ymax=1e+2,
	ymode=log,xmode=log,
	legend style={legend pos=south east, nodes=right},
	title={\(m=64\)}
	]
	\addplot[smooth,color=blue,line width=0.5pt,densely dashed] table[x=ell,y=matM64] {estimates_mat3d.txt};
	\addplot[smooth,color=blue,line width=0.5pt] table[x=ell,y=matM64] {measured_errors_mat3d.txt};
	\end{axis}
	\end{tikzpicture}
	\begin{tikzpicture}
	\begin{axis}[width=0.32\linewidth,
	ymin=1e-7,ymax=1e+6,
	ymode=log,xmode=log,
	legend style={legend pos=south east, nodes=right},
	xlabel={\(\ell\)},
	ylabel={\(\|\kappa_\text{ERR}^{\text{derm}}\|_\infty\)}
	]
	\addplot[smooth,color=blue,line width=0.5pt,densely dashed] table[x=ell,y=dermatM16] {estimates_mat3d.txt};
	\addplot[smooth,color=blue,line width=0.5pt] table[x=ell,y=dermatM16] {measured_errors_mat3d.txt};
	\end{axis}
	\end{tikzpicture}
	\begin{tikzpicture}
	\begin{axis}[width=0.32\linewidth,
	ymin=1e-7,ymax=1e+6,
	ymode=log,xmode=log,
	legend style={legend pos=south east, nodes=right},
	xlabel={\(\ell\)},
	]
	\addplot[smooth,color=blue,line width=0.5pt,densely dashed] table[x=ell,y=dermatM32] {estimates_mat3d.txt};
	\addplot[smooth,color=blue,line width=0.5pt] table[x=ell,y=dermatM32] {measured_errors_mat3d.txt};
	\end{axis}
	\end{tikzpicture}
	\begin{tikzpicture}
	\begin{axis}[width=0.32\linewidth,
	ymin=1e-7,ymax=1e+6,
	ymode=log,xmode=log,
	legend style={legend pos=south east, nodes=right},
	xlabel={\(\ell\)},
	]
	\addplot[smooth,color=blue,line width=0.5pt,densely dashed] table[x=ell,y=dermatM64] {estimates_mat3d.txt};
	\addplot[smooth,color=blue,line width=0.5pt] table[x=ell,y=dermatM64] {measured_errors_mat3d.txt};
	\end{axis}
	\end{tikzpicture}
	\caption[Comparison of the measured true Fourier approximation errors and the corresponding error bounds in three dimensions.]{Comparison of the measured true Fourier approximation errors (solid lines) for the periodically continued kernels $\kappa_{\text{R}}$ and the corresponding error estimators (dashed lines) as stated in Theorems~\ref{theorem:Fourier_Error_Matern_3d} and \ref{theorem:Fourier_Error_derMatern_3d} for the periodized kernels $\tilde{\kappa}$ in three dimensions.
		The results for the Mat\'ern\texorpdfstring{\((\frac{1}{2})\)}{} kernel are depicted in the first row and for the derivative Mat\'ern\texorpdfstring{\((\frac{1}{2})\)}{} kernel in the second row of the plot. The grid size $m$ is fixed to $m=16$, $m=32$, or $m=64$ (from left to right).
		\label{fig:error_estimates_mat3d}
	}
\end{figure}
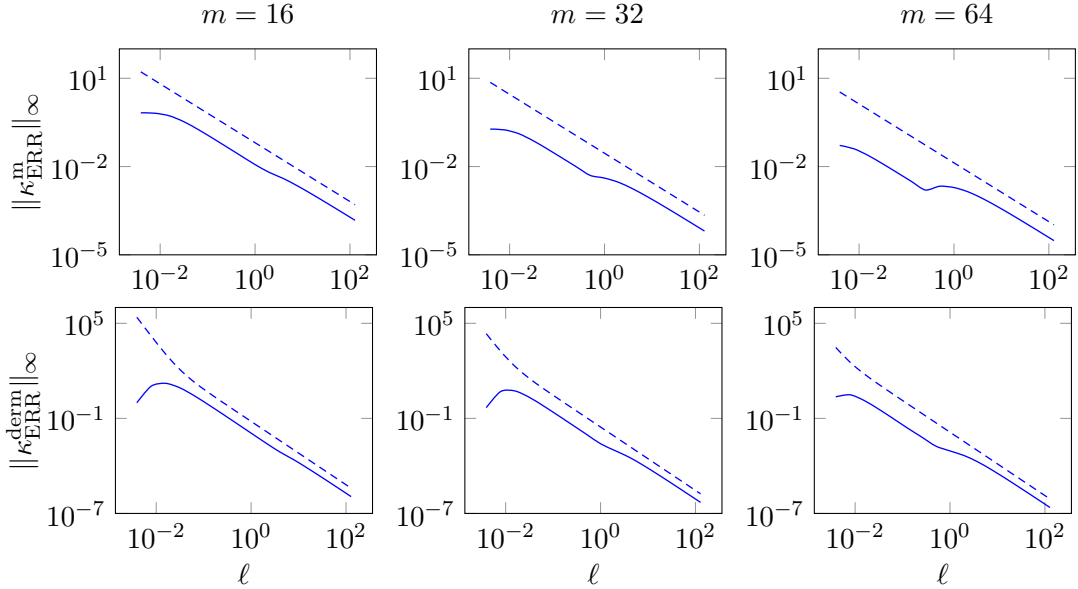

In Figure~\ref{fig:error_estimates_mat3d}, the corresponding results are presented for different values of $m$ and $\ell$, where the dashed lines represent the estimates established in Theorems~\ref{theorem:Fourier_Error_Matern_3d} and \ref{theorem:Fourier_Error_derMatern_3d} and the solid lines the measured errors as explained above. The error estimates can serve as good indicators for the true error, especially when $m\ell\pi>1$. For very small length-scales, choosing $m=16$ is not large enough, and the error estimator can be up to five orders of magnitude larger than the true error. For moderate length-scale values, the error estimator exceeds the true error by only one to two orders of magnitude. It should be noted that, even for the largest length-scales examined in this experiment, the error estimator remains a valid upper bound for the actual error (confer Remark~\ref{rem::bounds}). Note that for fixed $\ell$ and increasing $m$ the errors decrease, albeit slowly, which is supported by the estimates established above. 

\subsection{Generalization to Mat\'ern Kernels of Different Order}
The methods of periodization and NFFT-accelerated kernel operations are not limited to specific Mat\'ern kernels but can be extended to further frequently used kernels.
Our implementation relies on the NFFT-based fast summation algorithm, which is part of the publicly available NFFT software library~\cite{KeKuPo09}.
In addition to the Gaussian kernel and the Mat\'ern\texorpdfstring{\((\frac{1}{2})\)}{} kernel considered in this paper, this software also includes other kernels such as for example the (inverse) multiquadric \(\kappa(\bm x,\bm y)=\left(\|\bm x-\bm y\|^2+c^2\right)^{\pm1/2}\) or the sinc kernel \(\kappa(\bm x,\bm y)=\|\bm x-\bm y\|^{-1} \sin(c \|\bm x-\bm y\|)\), just to mention a few.
The well-known Mat\'ern kernels of various orders, defined in terms of the modified Bessel function of second kind, would also be suitable for the method and could easily be added to the list of possible kernels.
In addition, these more general kernels are also integrable and, thus, the periodization \eqref{eq:periodic_summation} is well defined, too.
We can therefore proceed in the same way in order to derive error estimates for these kernels.

\section{Numerical Experiments} \label{sec:Experiments}

In this section, we evaluate the accuracy and performance of the NFFT-accelerated additive GP method.
The proposed method has been implemented both in \texttt{C} and in \texttt{MATLAB}.
Our \texttt{C} implementation is based on the \texttt{FFTW} package \cite{frigo2005design} and the \texttt{NFFT} package \cite{KeKuPo09}.
We run our experiments on an Ubuntu 20.04.4 LTS machine equipped with 755 GB of system memory and a 24-core 3.0 GHz Intel Xeon Gold 6248R CPU.
Our \texttt{MATLAB} tests use \texttt{MATLAB} version R2024b.
Our code is compiled with the GCC 9.4.0 compiler and takes advantage of shared memory parallelism using \texttt{OpenMP}.
We use the parallel \texttt{BLAS} and \texttt{LAPACK} implementation in the \texttt{OpenBLAS} library for basic matrix operations.
We fixed the parameter $m$ in NFFT to $32$ in our experiments, and executed all our experiments in double precision.

\subsection{\texttt{AAFN} Preconditioned Iterations} \label{subsec:Preconditioned CG}

In our first set of experiments, we evaluate the performance of the \texttt{AAFN} preconditioned CG and stochastic trace estimation for both Gaussian kernel and Mat\'ern\texorpdfstring{\((\frac{1}{2})\)}{} kernel for different $\ell$.
For these experiments, we generated a synthetic dataset $\mathcal{X} \subset \mathbb{R}^6$. This dataset contains $3000$ points sampled uniformly at random within a hypercube of side length $\sqrt[3]{3000}$. The window was set to be $[[1,2,3],[4,5,6]]$ for simplicity.

We set the kernel hyperparameters $\sigma_f^2=\frac{1}{P}$ and noise variance $\sigma_\varepsilon^2=0.01$. The length scale $\ell$ was varied over a range emphasizing the ``middle rank'' regime, where convergence of the unpreconditioned CG is typically slow. For solving the resulting linear systems, the right-hand side vector had elements drawn uniformly from $[-0.5, 0.5]$. We compared the standard CG method against \texttt{AAFN} preconditioned CG. The \texttt{AAFN} preconditioner was configured with a maximum rank of $300$ and a maximum Schur complement fill level of $100$. Both solvers were iterated until the relative residual norm was reduced below a tolerance of $10^{-4}$, or up to a maximum of $200$ iterations.

Figure~\ref{fig:aafn-cg} shows that unpreconditioned CG converges rapidly when the length scale $\ell$ is either very small or very large. However, convergence becomes challenging for unpreconditioned CG over a wide intermediate range of $\ell$ for both the Gaussian and Mat\'ern kernels. The \texttt{AAFN} preconditioner significantly reduces the number of iterations required to achieve convergence across this challenging range, with the effect being particularly pronounced for the Mat\'ern kernel.

\begin{figure}[htbp]
    \centering
    \includegraphics[width=0.45\linewidth]{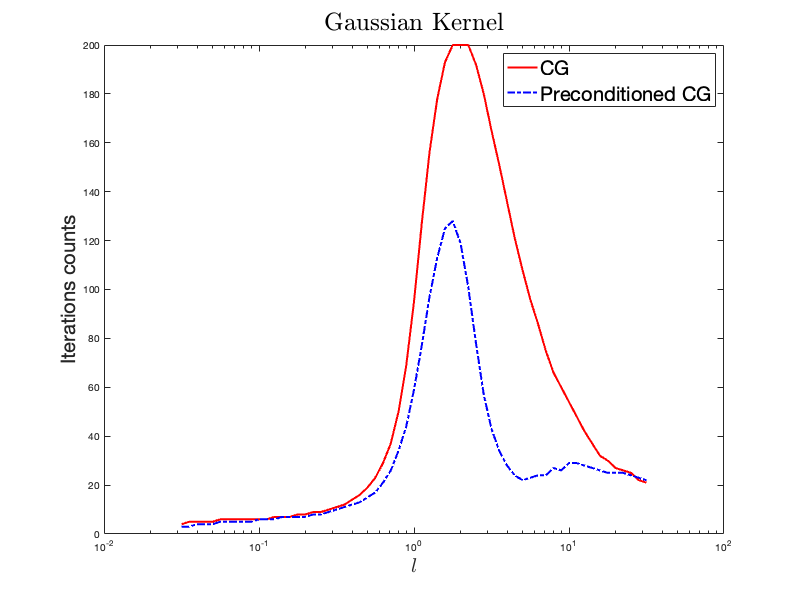}
    \includegraphics[width=0.45\linewidth]{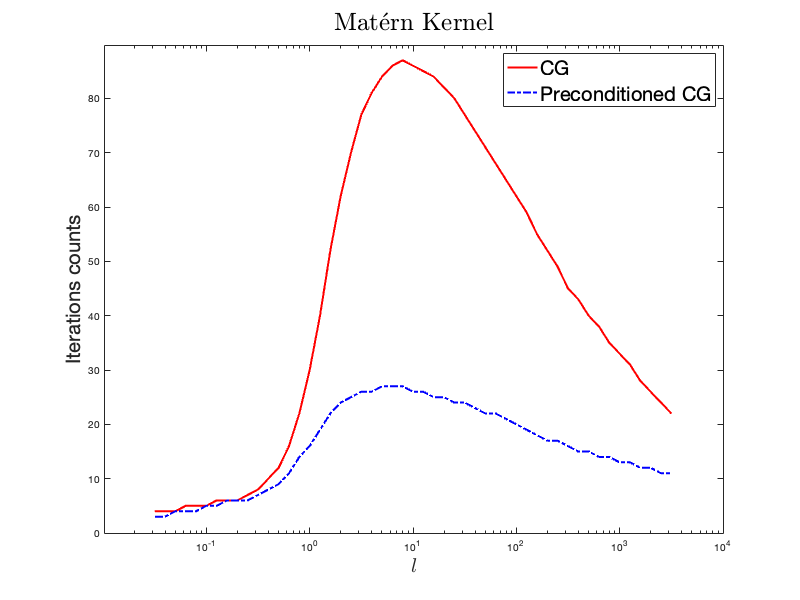}
    \caption{
    Comparison of iteration counts for CG vs \texttt{AAFN} preconditioned CG to reach $10^{-4}$ relative residual tolerance, as a function of length scale $\ell$. Plotted for Gaussian and Mat\'ern\texorpdfstring{\((\frac{1}{2})\)}{} kernels using a synthetic $\mathbb{R}^6$ dataset formed with $3000$ points sampled uniformly at random within a hypercube of side length $\sqrt[3]{3000}$. The right-hand side vector elements were sampled uniformly from $[-0.5, 0.5]$, using a zero initial guess.}
    \label{fig:aafn-cg}
\end{figure}

In the next set of experiments, we demonstrate that the \texttt{AAFN} preconditioner can eventually yield more accurate $\tilde{Z}(\bm{\theta})$ in \eqref{eq:GP_approx_objective} and its gradient in \eqref{eq:div}. The dataset $\mathcal{X} \subset \mathbb{R}^6$ was generated with $3000$ points, where each coordinate of $\bm{x} \in \mathcal{X}$ was sampled independently and uniformly from $[0,1]$. Labels $\bm{y}_i$ for each point $\bm{x}_i$ were generated as $\bm{y}_i = \sin(2\pi\bm{x}_i)^\intercal \exp(\bm{x}_i) + \|\bm{x}_i\|_2^2 + \varepsilon_i$, where $\sin(\cdot)$ and $\exp(\cdot)$ were applied element-wise, and $\varepsilon_i \sim \mathcal{N}(0, 0.01)$. 
The window was again set to be $[[1,2,3],[4,5,6]]$

We modeled this data using GPs with a Gaussian kernel, setting $\sigma_f^2=\frac{1}{P}$, $\sigma_\varepsilon^2=1.0$, and $\ell=2.0$ to ensure a ``middle rank'' kernel matrix. 
We used $5$ vectors in SLQ and Hutchinson trace estimator, and compared the mean and variance of the unpreconditioned version with the \texttt{AAFN} preconditioned one. We repeated experiments with iteration counts from $1$ to $10$. The \texttt{AAFN} preconditioner was configured with a maximum rank of $100$ and a maximum Schur complement fill level of $100$.

Figure~\ref{fig:aafn-loss} displays the mean estimate and $95\%$ confidence interval for  $\tilde{Z}(\bm{\theta})$ and its gradient $\partial \tilde{Z}/\partial \ell$ for different iteration counts. The results demonstrate that \texttt{AAFN} can significantly reduce the variance of the SLQ used in estimating the loss and stochastic trace estimation used in estimating its derivative.

\begin{figure}[htbp]
    \centering
    \includegraphics[width=0.45\linewidth]{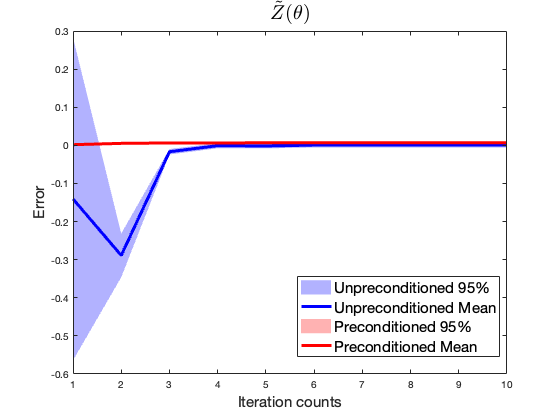}
    \includegraphics[width=0.45\linewidth]{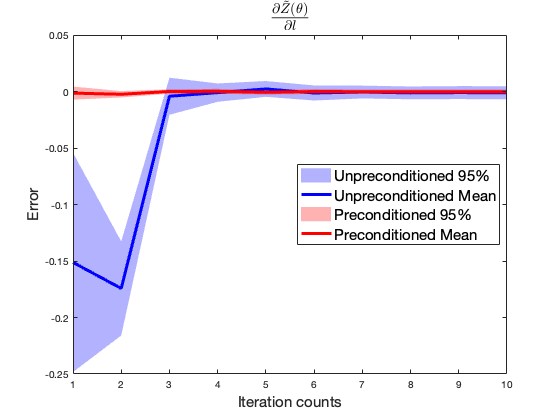}
    \caption{Comparison of the mean estimate and $95\%$ confidence interval for $\tilde{Z}(\bm{\theta})$ and its gradient $\partial \tilde{Z}/\partial \ell$ for different iteration counts. Plotted for Gaussian kernel using a synthetic $\mathbb{R}^6$ dataset formed with $3000$ points sampled uniformly from $[0,1]^6$. Labels $\bm{y}_i$ for each point $\bm{x}_i$ were generated as $\bm{y}_i = \sin(2\pi\bm{x}_i)^\intercal \exp(\bm{x}_i) + \|\bm{x}_i\|_2^2 + \varepsilon_i$, where $\sin(\cdot)$ and $\exp(\cdot)$ are applied element-wise, and $\varepsilon_i \sim \mathcal{N}(0, 0.01)$.}
    \label{fig:aafn-loss}
\end{figure}

\subsection{NFFT-accelerated Additive Kernel} \label{subsec:GP_framework} In this section, we compare the performance of four frameworks on both synthetic and real datasets: 1) GPs with one single kernel; 2) GPs with the additive kernel; 3) the SVGP approach  \cite{hensman2013gaussian,allison2023leveraging}; and 4)  GPs with the NFFT-accelerated additive kernel.
We choose to include SVGP as a baseline because it is one of the most widely used inducing point methods and is known to be very efficient for large problems. However, SVGP uses an approximated model and thus fails to provide accurate uncertainty compared to the exact GP model.
For the dataset containing more than $300,000$ data points in $\mathbb{R}^2$ or $\mathbb{R}^3$, we utilize a high-accuracy approximate matrix-vector multiplication method as detailed in \cite{huang2020h2pack,huang2025higp}, which serves as an alternative to exact matrix-vector multiplication for the comparison to the NFFT-accelerated method. In all GP experiments, we employ the Adam optimizer with a learning rate $0.01$ and a maximum iteration $500$ to train the hyperparameters.
Unless otherwise specified, our default settings are as follows: $10$ iterations with $10$ vectors for SLQ and stochastic trace estimation, $10$ CG iterations for training, $50$ CG iterations for prediction, and a fixed number of $10$ as the number of landmark points chosen for each sub-kernel in the \texttt{AAFN} preconditioner.
To ensure the positivity of all hyperparameters, we train them in $\mathbb{R}$ and apply the softplus function to transform them into the actual hyperparameters used in the kernels.  
Our initial guess for all three hyperparameters (before transformation) is zero.

In the first experiment, we evaluate the performance with a one-dimensional dataset without using the additive kernel to show that the NFFT-accelerated kernel operations can yield competitive prediction performance compared to the exact GPs.
We randomly sampled $1000$ points in $[0,1]$ and generated labels using a Gaussian Random Field with zero mean and Gaussian kernel covariance matrix with $\sigma_f^2=\frac{1}{P}$, $\ell=0.1$, and $\sigma_\varepsilon^2=0.01$.
Then, we randomly selected $800$ points for training and $200$ points for testing. The results for different model/method pairs are reported in Figure~\ref{fig:test2}. The results show that despite using approximate kernel operations, the loss curves and predictions remain consistent with those obtained with exact kernel operations for both Gaussian kernel and Mat\'ern\texorpdfstring{\((\frac{1}{2})\)}{} kernel.

\begin{figure}
    \centering
    \includegraphics[width=\textwidth]{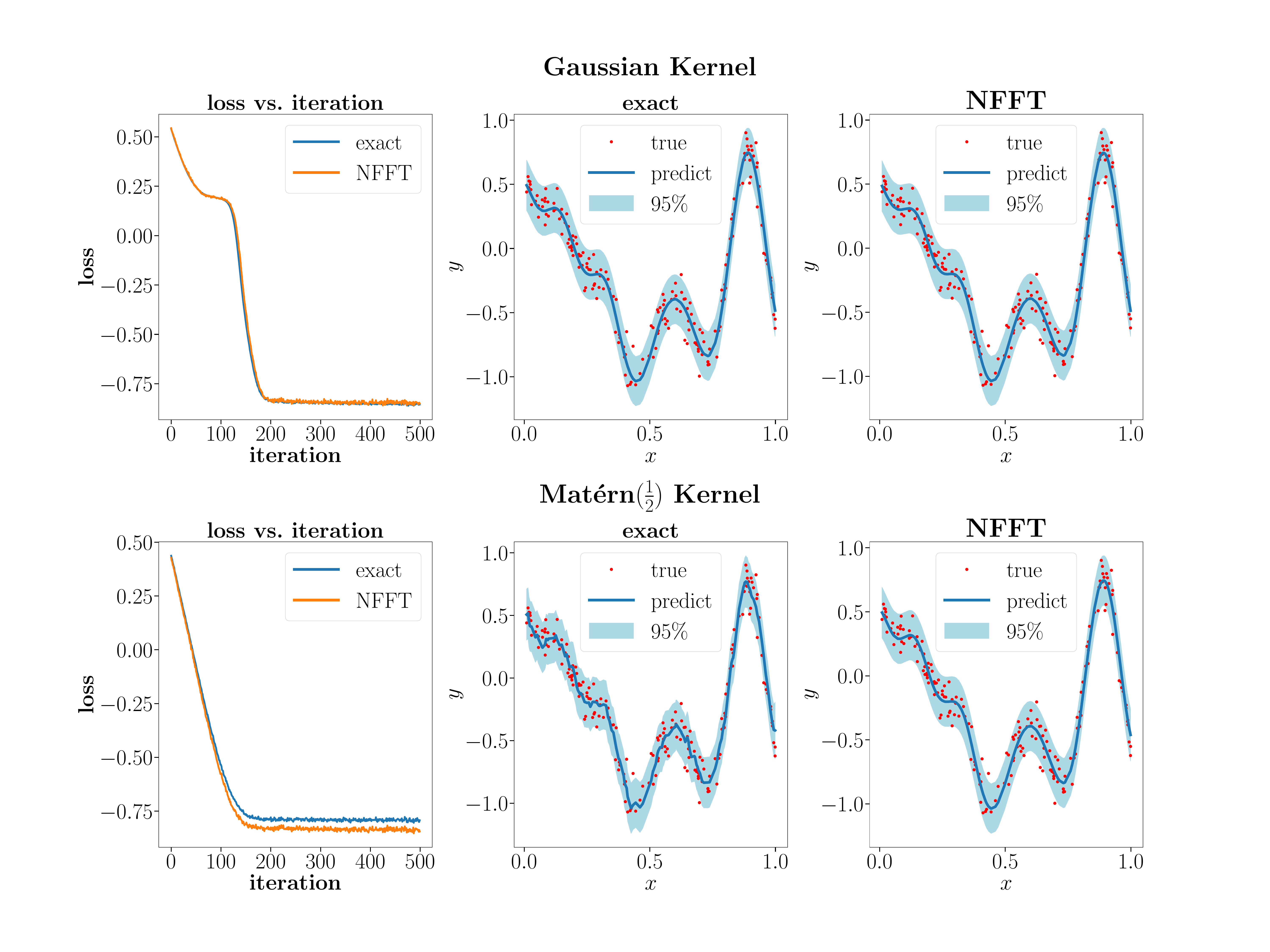}
    \caption{Comparision between NFFT-accelerated GPs and  GPs with exact matrix operations on a one-dimensional synthetic dataset. The labels are generated using a Gaussian Random Field. 
    }
    \label{fig:test2}
\end{figure}

Next, we test the performance of GPs with an additive kernel on a high-dimensional synthetic dataset.
Specifically, we randomly sampled $3000$ points in $\mathbb{R}^{20}$ and used the first six features to generate labels using a Gaussian Random Field with zero mean and Gaussian kernel covariance matrix, employing parameters $\sigma_f^2=\frac{1}{P}$, $\ell=1.0$, and $\sigma_\varepsilon^2=0.0001$. We employed the feature grouping technique EN, using $1000$ subsample points, with an $L_1$ regularization parameter set to $0.01$ and dimensionality $d=9$. This window selection process successfully identifies the correct features to form the feature windows $\mathcal{W}=\left[\left[6,4,5\right],\left[3,2,1\right]\right]$.
We then randomly selected $2400$ points for training and $600$ points for testing, and benchmarked our NFFT-accelerated GPs against the GPs using the exact additive kernel.

\begin{figure}
    \centering
    \includegraphics[width=\textwidth]{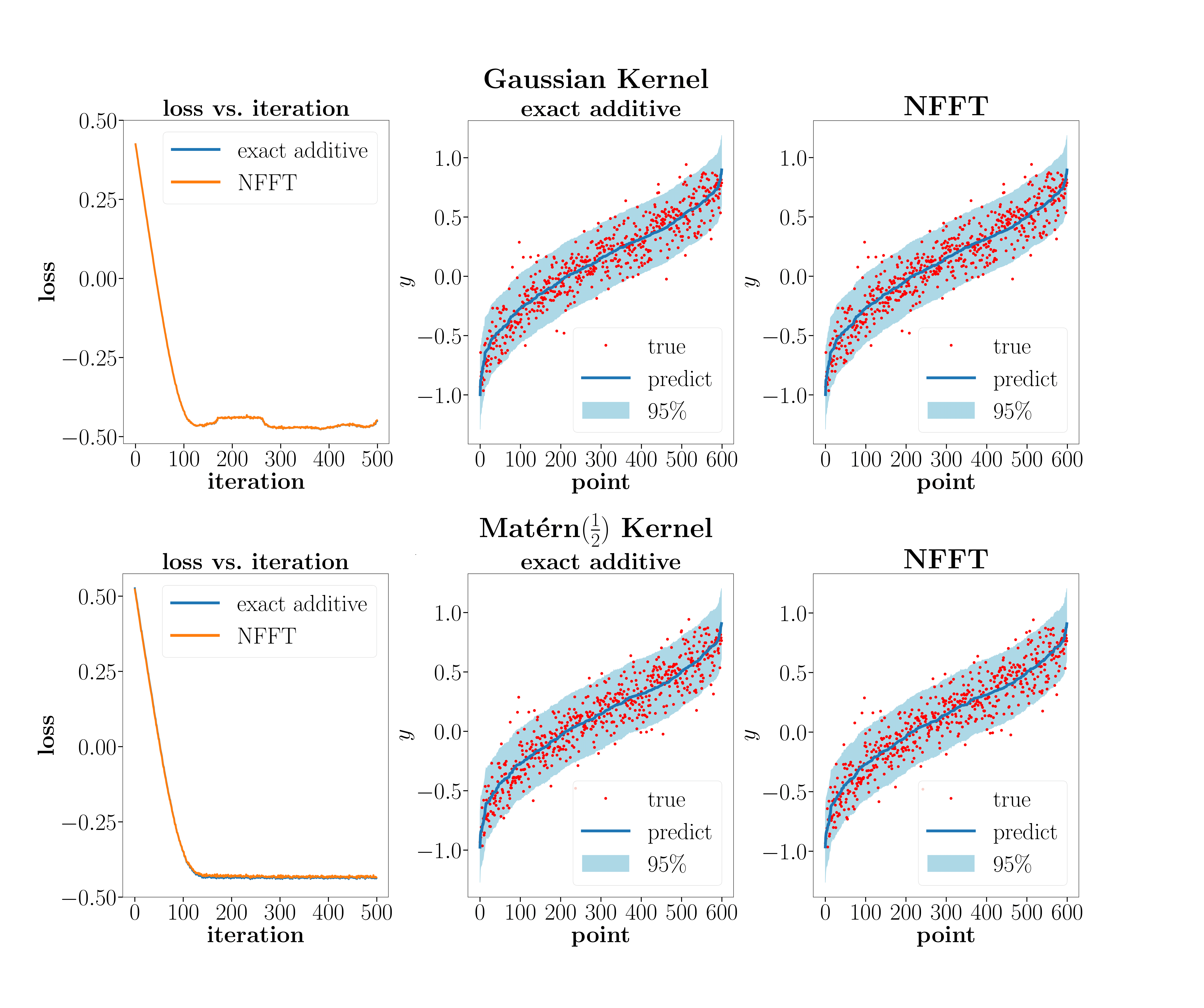}
    \caption{Comparison between GPs with the NFFT-accelerated additive kernel and GPs with the exact additive kernel on a high-dimensional synthetic dataset. Labels are generated using a Gaussian Random Field based on the first six features.} 
    \label{fig:test3}
\end{figure}

The results are plotted in Figure \ref{fig:test3}. 
In each plot, we plot the predictions as well as the 95\% confidence interval. 
The figure demonstrates that the additive kernel, after hyperparameter optimization, effectively captures the underlying pattern of the data with both kernels. 
The loss curves for the exact and NFFT-accelerated additive kernels closely align in the figure. 
We also tested the exact GPs with one single kernel. For the GPs based on one single kernel, the root mean square error (RMSE) values are $0.08$ and $0.12$ for the Gaussian kernel and Mat\'ern\texorpdfstring{\((\frac{1}{2})\)}{} kernel, respectively. 
For the GPs based on additive kernels, the final RMSEs  for both methods are close.
The RMSEs are $0.14$ and $0.15$ for the two kernels. These results reveal that GPs equipped with an additive kernel perform comparably to those using a single, exact kernel. Additionally, GPs utilizing a NFFT-accelerated additive kernel maintain accuracy similar to that of the exact additive kernel.

Finally, we ran large tests on some real datasets listed in Table \ref{tab:GP_final}. These datasets are from the UCI repository \cite{asuncion2007uci} and are widely used in GPs benchmarking, including the low-dimensional dataset \texttt{road3d} and the high-dimensional datasets \texttt{poletele}, \texttt{bike}, and \texttt{elevators}. The data are preprocessed as described in \cite{allison2023leveraging} except for \texttt{elevators}, which is downloaded from \texttt{GPyTorch} \cite{gardner2018gpytorch} GitHub repository. First, we analyze the impact of dimensionality reduction induced by the additive kernel structure on the GP model's performance. For this, we performed the NFFT-accelerated GP model with MIS feature grouping at feature ratios $d_{\text{ratio}}=\tfrac{1}{3}$, $d_{\text{ratio}}=\tfrac{2}{3}$, $d_{\text{ratio}}=1$ and compared its performance to that of the exact GP model utilizing a single kernel. The corresponding feature windows are given in Table~\ref{tab:GP_dratio_windows}. Due to the usage of a different kernel structure (additive vs.\ non-additive), a direct comparison of the negative log marginal likelihood with different methods is not meaningful. Therefore, we focus on reporting the RMSEs which are summarized in Table~\ref{tab:GP_dratio}.

\begin{table}[!ht]
	\centering
	\adjustbox{max width=0.8\textwidth}{
		\centering
		\begin{tabular}{|c|c||c|}
		\hline
		Dataset & $d_{\text{ratio}}$ & $\mathcal{W}$ \\ \hline \hline
		\multirow{3}{*}{\texttt{bike}} & $\tfrac{1}{3}$ & $[[2,7,9],[4,10]]$ \\
		& $\tfrac{2}{3}$ & $[[2,7,9],[4,10,5],[8,3,6]]$ \\
		& $1$ & $[[2,7,9],[4,10,5],[8,3,6],[1,11,12],[13]]$ \\ \hline
		\multirow{3}{*}{\texttt{elevators}} & $\tfrac{1}{3}$ & $[[10,11,12],[13,18,6]]$ \\
		& $\tfrac{2}{3}$ & $[[10,11,12],[13,18,6],[4,2,9],[3,8,16]]$ \\
		& $1$ & $[[10,11,12],[13,18,6],[4,2,9],[3,8,16],[7,5,14],[15,17,1]]$ \\ \hline
		\multirow{3}{*}{\texttt{poletele}} & $\tfrac{1}{3}$ & $[[1,2,4],[7,19,17],[3]]$ \\
		& $\tfrac{2}{3}$ & $[[1,2,4],[7,19,17],[3,12,15],[10,5,8],[14]]$ \\
		& $1$ & $[[1,2,4],[7,19,17],[3,12,15],[10,5,8],[14,9,18],[13,6,16],[11]]$ \\ \hline
	\end{tabular}}
	\caption{Feature windows $\mathcal{W}$ obtained from MIS feature grouping at different feature ratios $d_{\text{ratio}}$ for the GP models with the NFFT-accelerated  additive kernel.}
	\label{tab:GP_dratio_windows}
\end{table}

\begin{table}[!ht]
	\centering
	\begin{tabular}{|c||c|c|c|c||c|c|c|c|}
		\hline
		\multirow{2}{*}{Dataset} & 
		\multicolumn{4}{c||}{Gaussian Kernel} &
		\multicolumn{4}{c|}{Mat\' ern\texorpdfstring{\((\frac{1}{2})\)}{} Kernel} \\
		\cline{2-9}
		& $\tfrac{1}{3}$ & $\tfrac{2}{3}$ & $1$ & exact GPs & $\tfrac{1}{3}$ & $\tfrac{2}{3}$ & $1$ & exact GPs \\ \hline \hline
		\texttt{bike} & $0.61$ & $\bm{0.56}$ & $0.60$ & $0.59$ & $0.61$ & $\bm{0.56}$ & $0.58$ & $0.57$ \\ \hline
		\texttt{elevators} & $0.19$ & $0.12$ & $0.12$ & $\bm{0.09}$ & $0.19$ & $0.15$ & $0.15$ & $\bm{0.09}$ \\ \hline
		\texttt{poletele} & $0.29$ & $0.28$ & $0.29$ & $\bm{0.18}$ & $0.21$ & $0.21$ & $0.22$ & $\bm{0.16}$ \\ \hline
	\end{tabular}
	\caption{Comparison of the RMSE obtained from GP models with the NFFT-accelerated additive kernel using different MIS feature grouping ratios, against the exact GP model utilizing a single kernel.
	}
	\label{tab:GP_dratio}
\end{table}

By setting \(d_{\text{ratio}} < 1\), we effectively reduce the number of features used for prediction and consequently, the number of sub-kernels, which significantly speeds up computations involving the additive kernel. For the \texttt{poletele} dataset, employing ratios \(d_{\text{ratio}} = \tfrac{1}{3}\) and \(d_{\text{ratio}} = \tfrac{2}{3}\) yields similar RMSEs across both kernels, albeit higher than those obtained with the exact GP model with one single kernel. The \texttt{elevators} dataset shows that reducing dimensionality to \(d_{\text{ratio}} = \tfrac{2}{3}\) allows the RMSE of the additive Fourier-accelerated model to remain consistent, though still surpassing the RMSE of the exact GP model. Interestingly, in the \texttt{bike} dataset, while a \(d_{\text{ratio}} = \tfrac{1}{3}\) increases the RMSE compared to a full feature set (\(d_{\text{ratio}} = 1\)), a setting of \(d_{\text{ratio}} = \tfrac{2}{3}\) not only improves computational efficiency but also achieves the lowest RMSE, outperforming the exact GP model with one single kernel. Overall, while dimensionality reduction generally results in higher RMSEs compared to models using one single kernel, strategic settings of \(d_{\text{ratio}}\) can maintain or even enhance model performance.

We then alternatively employ the feature grouping technique EN next, define a target number of features $d_{\text{EN}}=9$ for the GPs with the NFFT-accelerated additive kernel, and compare their performance to exact GPs with one single kernel and the SVGP model's performance~\cite{allison2023leveraging}. The results are reported in Table~\ref{tab:GP_final}. Here, G refers to the Gaussian kernel and M to the Mat\'ern\texorpdfstring{\((\frac{1}{2})\)}{} kernel. The SVGP results are taken from the average values in \cite{allison2023leveraging}. The dataset \texttt{elevators} is not reported for SVGP in \cite{allison2023leveraging}. For \texttt{road3d}, as the dataset is very large, we use a high accuracy approximation to exact GPs available in \texttt{HiGP} \cite{huang2025higp} as discussed above. The derivative of the Mat\'ern\texorpdfstring{\((\frac{1}{2})\)}{} kernel is not supported in \texttt{HiGP} so we are unable to generate results for \texttt{road3d}. When employing EN feature grouping, the number of features incorporated into the additive model not only depends on the target number of features but also on the regularization parameter $\lambda_{\text{EN}}$. Thus, $d_{\text{EN}}$ is the target number of features but it does not always match the actual number of features in $\mathcal{W}$, since features with values below a certain tolerance are consistently excluded. 

\begin{table}[!ht]
	\adjustbox{max width=\textwidth}{
	\centering
	\begin{tabular}{|c||ccc|c|cc|cc|}
		\hline
		Dataset &
		\multicolumn{1}{c}{$n$} &
		\multicolumn{1}{c}{$p$} &
		\multicolumn{1}{c|}{$\mathcal{W}$} &
		\multicolumn{1}{c|}{SVGP G} &
		\multicolumn{1}{c}{Exact G} &
		\multicolumn{1}{c|}{Exact M} &
		\multicolumn{1}{c}{Additive G} &
		\multicolumn{1}{c|}{Additive M} \\
		\hline \hline
		\texttt{bike} &
		$13034$ &
		$13$ & $[[2,9,7],[10,5,3],[6,4,12]]$ &
		\multicolumn{1}{c|}{$0.61$} &
		\multicolumn{1}{c}{$0.59$} &
		\multicolumn{1}{c|}{$0.57$} &
		\multicolumn{1}{c}{$0.63$} &
		\multicolumn{1}{c|}{$\bm{0.54}$} \\
		\hline
		\texttt{elevators} &
		$13279$ &
		$18$ & $[[10,8,13],[18,1,12],[11,3,4]]$ &
		\multicolumn{1}{c|}{-} &
		\multicolumn{1}{c}{$0.09$} &
		\multicolumn{1}{c|}{$\bm{0.09}$} &
		\multicolumn{1}{c}{$0.16$} &
		\multicolumn{1}{c|}{$0.16$} \\
		\hline
		\texttt{poletele} &
		$4406$ &
		$19$ & $[[1,2,4],[3,5]]$ &
		\multicolumn{1}{c|}{$0.23$} &
		\multicolumn{1}{c}{$0.18$} &
		\multicolumn{1}{c|}{$0.16$} &
		\multicolumn{1}{c}{$0.15$} &
		\multicolumn{1}{c|}{$\bm{0.13}$} \\
		\hline
		\texttt{road3d} &
		$326155$ &
		$2$ & - &
		\multicolumn{1}{c|}{$\bm{0.44}$} &
		\multicolumn{1}{c}{$0.69$} &
		\multicolumn{1}{c|}{-} &
		\multicolumn{1}{c}{$0.68$} &
		\multicolumn{1}{c|}{$0.62$} \\
		\hline
	\end{tabular}}
	\caption{Comparison of the RMSE between different methods. The feature grouping in the GP models with the NFFT-accelerated additive kernel is determined via EN, with a target number $d_{\text{EN}}=9$ of features and $\lambda_{\text{EN}}=0.01$.
G represents Gaussian kernel and M represents Mat\'ern\texorpdfstring{\((\frac{1}{2})\)}{} kernel. 
SVGP results are taken from the average values in \cite{allison2023leveraging}.
The dataset \texttt{elevators} is not reported for SVGP in \cite{allison2023leveraging}.
Derivative of Mat\'ern\texorpdfstring{\((\frac{1}{2})\)}{} kernel is not implemented in \texttt{HiGP} \cite{huang2025higp} so the results for \texttt{road3d} are not computed for Exact M.
	}
	\label{tab:GP_final}
\end{table}

The results presented in Table~\ref{tab:GP_final} demonstrate that the GP models with the NFFT-accelerated additive kernel achieve comparable performance to other methods while only relying on a much smaller number of features. While for \texttt{poletele}, the additive Fourier-accelerated model with MIS-based feature windows could not keep up with the RMSEs achieved by the exact model as shown in Table~\ref{tab:GP_dratio}, the Fourier-accelerated models based on windows generated via EN feature grouping now yield RMSEs which are smaller than the exact model's and SVGP's. For \texttt{bike}, the EN feature grouping could reduce the RMSEs for the Mat\'ern\texorpdfstring{\((\frac{1}{2})\)}{} kernel. In this experiment, the additive Fourier-accelerated models with the Mat\'ern\texorpdfstring{\((\frac{1}{2})\)}{} kernel achieved smaller RMSEs than with the Gaussian kernel.

\section{Conclusion} \label{sec:Conclusion}
We showed in this paper that Fourier-accelerated additive GP models can achieve similar RMSEs and consistent uncertainty compared to the exact models with exact or additive kernels. Our major focus was on introducing the additive Fourier approach to drastically reduce the computational complexity by inducing dimensionality reduction. We illustrated that the use of feature windows can produce competitive predictions and provided theoretical guidance on how to choose the correct number of windows based on a suitable feature grouping technique. A cornerstone of efficient optimization techniques is the fast and reliable computation of gradients of the objective function. We illustrated that this can also be done with the help of the NFFT-based multiplication, where not only accurate approximations to the kernel vector product are provided, but also the exact gradients to these approximations, which in turn allows us to work with the true gradients. We theoretically underpinned this process by providing rigorous error bounds for the Mat\'ern kernel.  The success of our numerical experiments also depends on the use of an efficient preconditioner, and we showed how the modification of the adaptive factorized Nystr\"om preconditioner greatly benefits the solution of the linear systems as well as the trace estimation procedure. This illustrates that using Fourier acceleration for matrix-vector products and preconditioning within GP training reduces memory consumption, shows greater potential for parallel computations, and provides more robustness towards changes in the hyperparameters during the training process.

In the future, the flexibility of our method should allow for the incorporation of more sophisticated optimization solvers in the training procedure.

\appendix

\section{The FFT for Nonequispaced Data (NFFT)}
\label{sec:nffti}
Consider a trigonometric polynomial of the form
$$
f(\bm x)=\sum_{\bm k\in\mathcal I_m} b_{\bm k} \,\e^{2\pi\i\bm k^\intercal\bm x},
$$
which we would like to evaluate in a set of points $\bm x_j\in\mathbb T^d =\mathbb R^d/\mathbb Z^d \simeq[-\tfrac12,\tfrac12)^d$, $j=1,\dots,n$.
While this evaluation can be realized efficiently in $\mathcal O(|\mathcal I_m| \log \mathcal |I_m|)$ operations in case the given data points $\bm x_j$ sit on a regular grid, it is not obvious how this can be generalized to arbitrary data points, which we are regularly confronted with in practical issues.

The basic idea of the NFFT, see \cite{potts2003fast,KeKuPo09} and references therein, is to approximate the given function $f$ as a sum of translates
\begin{equation}\label{eq:nfft_f_approx}
    f(\bm x) \approx \sum_{\bm l\in\mathcal I_{\sigma m}} g_{\bm l} \tilde\varphi\left(\bm x-\tfrac{\bm l}{\sigma m}\right)=: f_\text{nfft}(\bm x),
\end{equation}
where $\tilde\varphi:\mathbb T^d\to\mathbb R$ is 1-periodic function, constructed in terms of the 1-periodic periodization of a window function 
$$
\varphi: \mathbb T^d\to \mathbb R, \quad \text{with } \mathrm{supp}(\varphi)=\left[-\tfrac{s}{\sigma m},\tfrac{s}{\sigma m}\right]  
$$
having a small support. Thus, for a given point $\bm x$ only a few summands in \eqref{eq:nfft_f_approx} will be non-zero, so that it can be evaluated within a small number of arithmetic operations.
The parameter $s\in\mathbb N$, $s\ll\sigma m$ is called the support parameter of the window function and $\sigma\geq1$ is an oversampling factor.
The (up to now unknown) coefficients $g_{\bm l}$ in \eqref{eq:nfft_f_approx} have to be set depending on the given Fourier coefficients $b_{\bm k}$ and the chosen window function $\varphi$. This can be done as follows.

Note that the right-hand side in \eqref{eq:nfft_f_approx} is a type of a cyclic convolution. Thus, computing the Fourier coefficients of the functions on both sides of the equation gives
$$
b_{\bm k} \approx \text{FFT}([g_{\bm l}]_{\bm l\in\mathcal I_m})_{\bm k}\cdot c_{\bm k}(\tilde\varphi) \approx \text{FFT}([g_{\bm l}]_{\bm l\in\mathcal I_m})_{\bm k}\cdot c_{\bm k}(\varphi),
$$
where we assume that the window function $\varphi$ is also well-localized in the frequency domain and, thus, based on the aliasing formula we have $c_{\bm k}(\tilde\varphi)\approx c_{\bm k}(\varphi)$.
Finally, the NFFT workflow can be summarized as follows.
\begin{enumerate}
    \item Compute
    $$
    \tilde b_{\bm k}:=
    \begin{cases}
        b_{\bm k}\cdot c_{\bm k}(\varphi)^{-1} &: \bm k\in\mathcal I_m, \\
        0 &: \bm k\in \mathcal I_{\sigma m}\setminus \mathcal I_m.
    \end{cases}
    $$
    \item Compute the coefficients $g_{\bm l}$, $\bm l\in\mathcal I_{\sigma m}$, by applying an (ordinary) FFT to the coefficients $\tilde b_{\bm k}$, $\bm k\in\mathcal I_{\sigma m}$.
    \item Evaluate the sparse sums \eqref{eq:nfft_f_approx} for all given data points $\bm x_j$, $j=1,\dots,n$.
\end{enumerate}
As can be seen from the above explanations, the NFFT is an approximate algorithm, that is, the given trigonometric polynomial is not exactly evaluated.
The resulting approximation error heavily depends on the given Fourier coefficients, the applied window function and also on the parameters $s$ and $\sigma$.
However, by choosing the last two mentioned parameters large enough, the NFFT can be tuned to an arbitrary precision.

Within the NFFT software library, a number of window functions are implemented, confer \cite{KeKuPo09}. The default window function used is the so-called Kaiser-Bessel window function
$$
\varphi(x)\coloneqq \frac1\pi
\begin{cases}
    \dfrac{\sinh\left(\pi(2-\frac1\sigma) \sqrt{s^2-\sigma^2m^2x^2} \right)}{\sqrt{s^2-\sigma^2m^2x^2}}
    &: |x|\leq\frac{s}{\sigma m}, \\
    \dfrac{\sin\left(\pi(2-\frac1\sigma) \sqrt{\sigma^2m^2x^2-s^2} \right)}{\sqrt{\sigma^2m^2x^2-s^2}}
    &: \text{else},
\end{cases}
$$
where the second part is truncated, that is, the window is restricted to $[-\frac{s}{\sigma m},\frac{s}{\sigma m}]$, as explained above.
The Fourier coefficients of the periodized window $c_k(\tilde\varphi)$ are known explicitly in terms of the modified zero-order Bessel function.
The stated window function is the univariate version, i.e., $d=1$.
In higher dimensions, a tensor product approach is applied, where the multivariate window function is simply obtained by multiplying $d$ univariate functions. Thus, also the Fourier coefficients possess such a tensor product structure. 

In $d=1$ dimension, the approximation error can be estimated by \cite{plonka2018numerical}
\begin{equation}\label{eq:nfft_error}
    |f(x_j)-f_\text{nfft}(x_j)| \leq \|\bm b\|_1 \cdot
    4\pi (s+ \sqrt s)\sqrt[4]{1-\frac1\sigma} \e^{-2\pi s\sqrt{1-1/\sigma}},
\end{equation}
where the oversampling factor $\sigma$ is assumed to be $>1$ and by $\bm b$ we denote the vector containing the given Fourier coefficients $b_{\bm k}$.
We can see that the error decreases exponentially with the growing support parameter $s$.
Following the tensor product approach in higher dimensions, this estimate can be generalized accordingly.
Although there is no error estimate for the $d$-dimensional Kaiser-Bessel window available in the literature, results for other window functions \cite{Elbel1989} show that the error behaves the same way as in the 1-dimensional setting. The constant then grows like $d\cdot2^{d-1}$, that is, for $d=3$ only moderate differences are expected compared to the 1D estimation.

\section{Proof of Lemma~\ref{lemma1} \label{lemma1proof}}
First, we consider the univariate case with $\kappa^\text{m}_1(r):=\e^{-|r|/\ell}$, $x\in\mathbb R$.
Due to the fact that this function is decreasing with $|r|\to\infty$, the maximum error is obtained for $r=\pm \frac 12$, see also Figure~\ref{fig:1periodization}, which we can estimate by
\begin{align}
    \max_{r\in[-\frac12,\frac 12]} \left| \kappa^\text{m}_1(r)-\tilde \kappa^\text{m}_1(r)\right|
    &= \sum_{n\in\mathbb Z\setminus\{0\}} \e^{-\left|\frac12+n\right|/\ell}
    = \e^{-1/(2\ell)}+2\sum_{n=1}^\infty \e^{-(2n+1)/(2\ell)} \notag \\
    &\leq \e^{-1/(2\ell)}+2\,\e^{-1/(2\ell)}\int_0^\infty \e^{-x/\ell}\,\mathrm dx
    = \e^{-1/(2\ell)}(1+2\ell). \label{eq:error_f_ftilde_1d}
\end{align}
In the trivariate case, we use $\|\bm r\|_2\geq \frac1{\sqrt3}\|\bm r\|_1=\frac1{\sqrt3}(|r_1|+|r_2|+|r_3|)$ and obtain
\begin{align}
    |\kappa^\text{m}(\bm r)-\tilde \kappa^\text{m}(\bm r)|
    &= \sum_{\bm n\in\mathbb Z^3\setminus\{\bm 0\}} \e^{-\|\bm r+\bm n\|_2/\ell} \notag\\
    &\leq \sum_{\bm n\in\mathbb Z^3\setminus\{\bm 0\}} \e^{-\|\bm r+\bm n\|_1/(\sqrt3\ell)}
    =\left[\prod_{j=1}^3 \left(\sum_{n\in\mathbb Z} \e^{-|r_j+n|/(\sqrt3\ell)}\right)\right] - \e^{-\|\bm r\|_1/(\sqrt3\ell)}. \notag
\end{align}
By using \eqref{eq:error_f_ftilde_1d}, we obtain the estimate
\begin{equation*}
    \left| \kappa^\text{m}(\bm r)-\tilde\kappa^\text{m}(\bm r)\right|
    \leq \left[\prod_{j=1}^3\left( \e^{-|r_j|/(\sqrt3\ell)} + \e^{-1/(2\sqrt3\ell)}(1+2\sqrt3\ell) \right)\right] - \e^{-\|\bm r\|_1/(\sqrt3\ell)},
\end{equation*}
where the term $\e^{-\|\bm r\|_1/(\sqrt3\ell)}$ vanishes after splitting up the product into its single additive components.
After that, we obtain the assertion $|\kappa^\text{m}(\bm r)-\tilde \kappa^\text{m}(\bm r)|\leq\delta^\text{m}(\ell)$ by estimating $\e^{-|r_j|/(\sqrt3\ell)}\leq 1$. 

    
\section{Proof of Lemma~\ref{lemma2} \label{lemma2proof}}
The proof follows a similar line as the proof of Lemma~\ref{lemma1}.
First, we consider the univariate case with $\kappa_1(r)\coloneqq \frac{|r|}{\ell} \e^{-|r|/\ell}$ and note that this function has its global maxima in the points $r=\pm\ell$ with function values $\kappa_1(\pm\ell)=1$.
If these points lie in the interval $(-\frac12,\tfrac12)$, that is, $\ell<\frac12$, the maximum error between $\kappa_1$ and $\tilde\kappa_1$ is again obtained for $r=\pm\frac12$.
Then, we obtain
\begin{align*}
    \max_{r\in[-\frac12,\frac12]} |\kappa_1^\text{derm}(r)-\tilde\kappa_1^\text{derm}(r)|
    & = \frac{1}{\ell^2}\sum_{n\in\mathbb Z\setminus\{0\}} \left|\tfrac12+n\right| \,\e^{-|\frac12+n|/\ell} \\
    & = \frac1{\ell^2}\left(\frac12\e^{-1/(2\ell)}+ 2\sum_{n=1}^\infty \frac{2n+1}{2}\e^{-(2n+1)/(2\ell)}\right) \\
    & \leq \frac1{\ell^2}\left( \frac{\e^{-1/(2\ell)}}{2} + \e^{-1/(2\ell)} \int_0^\infty (2x+1)\e^{-x/\ell}\dx\right) \\
    & = \frac{\e^{-1/(2\ell)}}{2\ell^2} + \frac{\ell(2\ell+1)\e^{-1/(2\ell)}}{\ell^2}
    = \frac{\e^{-1/(2\ell)}(1+2\ell+4\ell^2)}{2\ell^2}.
\end{align*}
In the trivariate setting, we simply make use of
$$
\|\bm x+\bm n\|_2 \e^{-\|\bm x+\bm n\|_2/\ell}
\leq \|\bm x+\bm n\|_1 \e^{-\|\bm x+\bm n\|_1/(\sqrt3 \ell)}
$$
in order to split up the 3d-dimensional terms into products of univariate terms of the form
$$
a_j = a_j(\ell) = \sum_{n\in\mathbb Z\setminus\{0\}} |x_j+n| \e^{-|x_j+n|/(\sqrt3\ell)}
\leq \frac{\e^{-1/(2\sqrt3\ell)}}{2}(1+2\sqrt3\ell+12\ell^2)
$$
and
$$
b_j = b_j(\ell) = \sum_{n\in\mathbb Z\setminus\{0\}} \e^{-|x_j+n|/(\sqrt3\ell)}
\leq\e^{-1/(2\sqrt3\ell)}(1+2\sqrt3\ell).
$$
Consequently, we have
\begin{align*}
&\displaystyle\sum_{\bm n\in\mathbb Z\setminus\{\bm 0\}} \|\bm x+\bm n\|_1\e^{-\|\bm x+\bm n\|_1/(\sqrt3\ell)}
= \sum_{\bm n\in\mathbb Z} \|\bm x+\bm n\|_1\e^{-\|\bm x+\bm n\|_1/(\sqrt3\ell)}-\|\bm x\|_1\e^{-\|\bm x\|_1/(\sqrt3\ell)}\\
=& \sum_{\bm n\in\mathbb Z^3}|x_1+n_1|\e^{-\|\bm x\|_1/(\sqrt3\ell)}
+\ldots
+\sum_{\bm n\in\mathbb Z^3}|x_3+n_3|\e^{-\|\bm x\|_1/(\sqrt3\ell)}
-\|\bm x\|_1\e^{-\|\bm x\|_1/(\sqrt3\ell)}
\\
=&\left(|x_1|\e^{-|x_1|/(\sqrt3\ell)}+a_1\right)\!\!
\left(\e^{-|x_2|/(\sqrt3\ell)}+b_2\right)\!\!
\left(\e^{-|x_3|/(\sqrt3\ell)}+b_3\right)
+\ldots+\ldots
-\|\bm x\|_1\e^{-\|\bm x\|_1/(\sqrt3\ell)},
\end{align*}
where the term $\|\bm x\|_1\e^{-\|\bm x\|_1/(\sqrt3\ell)}$ is finally canceled out.
Using $\e^{-|\cdot|/\ell}\leq 1$, $\frac{|\cdot|}{\ell}\e^{-|\cdot|/\ell}\leq1$ and exploiting the underlying symmetry, we obtain the assertion.
\vfill

\end{document}